\documentclass[10pt,journal,compsoc]{IEEEtran}

\usepackage{times}

\usepackage{times}
\usepackage{epsfig}

\usepackage[colorlinks]{hyperref}
\usepackage{url}
\usepackage[utf8]{inputenc} 
\usepackage[T1]{fontenc}    
\usepackage{booktabs}       
\usepackage{amsfonts}       
\usepackage{nicefrac}       
\usepackage{microtype}      
\usepackage{xcolor}         
\usepackage{caption}
\usepackage{wrapfig}
\usepackage{amsthm,amssymb,amsmath,bm}
\usepackage{subcaption,graphicx}
\usepackage{listings}
\usepackage{multirow}
\usepackage{bbm}

\usepackage[compress]{cite}

\usepackage[vlined,boxed,commentsnumbered,ruled]{algorithm2e}
\definecolor{codeAnnotation}{RGB}{106,153,85}
\definecolor{citeColor}{RGB}{255,128,0}
\newtheorem{theo}{Theorem}[section]
\newtheorem{assumption}{Assumption}[section]
\newtheorem{lemma}{Lemma}[section]
\newtheorem{corollary}{Corollary}[section]
\newtheorem{definition}{Definition}[section]
\allowdisplaybreaks[4]

\allowdisplaybreaks[4]

\usepackage{pifont}
\newcommand{\xmark}{\ding{55}}

\usepackage{ragged2e}

\begin{document}

\title{Decouple Graph Neural Networks: Train Multiple Simple GNNs Simultaneously Instead of One}

\author{Hongyuan Zhang, Yanan Zhu,  
and Xuelong Li$^*$, \IEEEmembership{~Fellow,~IEEE} \thanks{$^*$ Corresponding author}

\thanks{This work is supported by The National Natural Science Foundation of China (No. 61871470).}

\thanks{Hongyuan Zhang and Yanan Zhu are with the School of Artificial Intelligence, OPtics and ElectroNics (iOPEN), Northwestern Polytechnical University, Xi'an 710072, P.R. China. 
\noindent Hongyuan Zhang and Xuelong Li are also with the Institute of Artificial Intelligence (TeleAI), China Telecom Corp Ltd, 31 Jinrong Street, Beijing 100033, P. R. China. 
}

\thanks{E-mail: hyzhang98@gmail.com, xuelong\_li@ieee.org.}

\thanks{
    \copyright 2022 IEEE.  Personal use of this material is permitted.  Permission from IEEE must be obtained for all other uses, in any current or future media, including reprinting/republishing this material for advertising or promotional purposes, creating new collective works, for resale or redistribution to servers or lists, or reuse of any copyrighted component of this work in other works.
}
\thanks{The source codes are available at \url{https://github.com/hyzhang98/SGNN}.}

}

\markboth{IEEE TRANSACTIONS ON PATTERN ANALYSIS AND MACHINE INTELLIGENCE}{Zhang \MakeLowercase{\textit{et al.}}: 
Decouple Graph Neural Networks: Train Multiple Simple GNNs Simultaneously Instead of One}

\IEEEtitleabstractindextext{
\justifying  
\begin{abstract}
    Graph neural networks (GNN) suffer from severe inefficiency 
    due to the exponential growth of node dependency with the increase of layers. 
    It extremely limits the application of stochastic optimization algorithms 
    so that the training of GNN is usually time-consuming. 
    To address this problem, we propose to decouple a multi-layer GNN as multiple simple modules 
    for more efficient training, which is comprised of classical \textit{forward training} (\textit{FT})
    and designed \textit{backward training} (\textit{BT}). 
    Under the proposed framework, each module can be trained efficiently in FT by stochastic 
    algorithms without distortion of graph information owing to its simplicity. 
    To avoid the only unidirectional information delivery of FT and 
    sufficiently train shallow modules with the deeper ones, 
    we develop a backward training mechanism that makes the former 
    modules perceive the latter modules, 
    inspired by the classical backward propagation algorithm. 
    The backward training introduces the reversed information delivery into 
    the decoupled modules as well as the forward information delivery. 
    To investigate how the decoupling and greedy training affect the representational 
    capacity, 
    we theoretically prove that the error produced by linear modules will 
    not accumulate on unsupervised tasks in most cases. 
    The theoretical and experimental results show that the proposed 
    framework is highly efficient with reasonable performance, 
    which may deserve more investigation. 
\end{abstract}


}

\maketitle

\section{Introduction}

In recent years, neural networks \cite{VGG,ResNet}, due to the impressive performance, have been 
extended to graph data, known as graph neural networks (GNNs) \cite{GNN}. 
As GNNs significantly improve the results of graph tasks, 
it has been extensively investigated from different aspects, such as 
graph convolution network (GCN) \cite{GCN,patchysan}, graph attention networks (GATs) \cite{GAT,GAT-2}, 
spatial-temporal GNN (STGNN) \cite{STGNN}, graph auto-encoder \cite{GAE,GALA},
graph contrastive learning \cite{MultiViewContrastive}, \textit{etc}. 

Except for the variants that originate from different perspectives, 
an important topic is motivated by the well-known inefficiency of GNN. 
In classical neural networks \cite{ResNet}, 
the optimization is usually based on stochastic algorithms with limited  
batch \cite{AdaGrad,Adam} since samples are independent of each other. 
However, the aggregation-like operations defined in \cite{GraphSAGE} 
result in the dependency of each node on its neighbors and the 
amount of dependent nodes for one node increases exponentially 
with the growth of layers, which results in the unexpected increases of batch size. 
Some works are proposed based on neighbor sampling \cite{GraphSAGE,FastGCN,StoGCN,GraphSAINT} 
and graph approximation \cite{Cluster-GCN} to limit the batch size, 
while some methods \cite{SGC,S2GC} attempt to directly apply high-order graph operation 
and sacrifice the most non-linearity. 
The training stability is a problem for neighbor sampling methods \cite{GraphSAGE,FastGCN,GraphSAINT} though VRGCN \cite{StoGCN} has attempted 
to control the variance via improving sampling. 
Note that the required nodes may still grow (slowly) with the increase of depth. 
Cluster-GCN \cite{Cluster-GCN} finds an approximate graph with plenty of 
connected components so that the batch size is strictly upper-bounded. 
The major challenge of these methods is the information missing
during sampling. 
The simplified methods \cite{SGC,S2GC} are efficient but the limited 
non-linearity may be the bottleneck of these methods. 
These methods may incorporate the idea of GIN \cite{GIN} 
to improve the capacity \cite{S2GC}.

To apply stochastic optimization while retaining the exact graph structure, 
we propose a framework, namely stacked graph neural network (SGNN), 
which decouples a multi-layer GNN as multiple simple GNN modules 
and then trains them simultaneously rather than connecting them 
with the increase of the depth. 
Inspired by the backward propagation algorithm, we find that the main 
difference between stacked networks \cite{SAE} and classical networks is 
\textit{no training information propagated from the latter modules to the former ones}. 
The lack of backward information delivery may be the main reason of the 
performance limitation of stacked models. 
The contributions are concluded as: 
\textbf{(1)}
We accordingly propose a backward training strategy to let the former modules 
receive the information from the final loss and latter modules, 
which leads to a cycled training framework to control bias and train shallow 
modules correctly. 
\textbf{(2)}
Under this framework, a multi-layer GNN can be decoupled into multiple simple 
GNNs, named as separable GNNs in this paper, so that every training step 
could use the stochastic optimization without any samplings or changes on graph. 
Therefore, SGNN could take both non-linearity and high efficiency into account.
\textbf{(3)}
We investigate how the decoupling and greedy training affect the representational 
capacity of the linear SGNN. 
It is proved that the error would not accumulate in most cases when the final 
objective is graph reconstruction. 

\begin{figure*}
    \centering
    \includegraphics[width=0.95\linewidth]{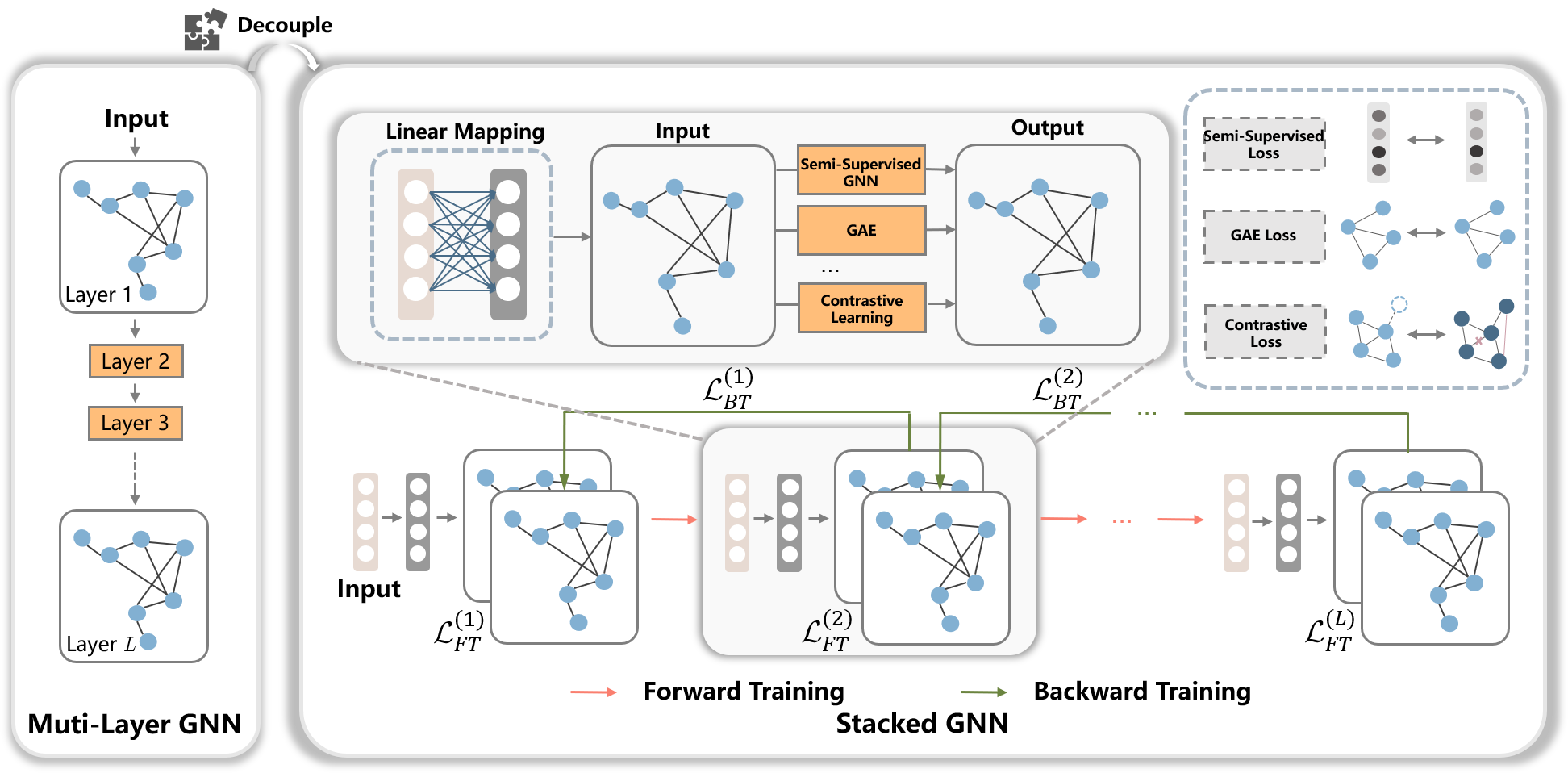}
    \caption{Illustration of a stacked graph neural network decoupled from an $L$-layer GNN. 
    To train each module individually, some loss (\textit{e.g.}, semi-supervised loss, unsupervised loss, contrastive loss) 
    is required and it is denoted by $\mathcal{L}_{FT}^{(t)}$. To let the shallow modules perceive the deeper 
    ones, $\mathcal{M}_t$ passes back the \textit{expected} input features to $\mathcal{M}_{t-1}$ 
    during the backward training. The divergence between the features output by $\mathcal{M}_t$ and the expected 
    features of $\mathcal{M}_{t-1}$ formulates the BT loss $\mathcal{L}_{BT}^{(t)}$. } 
    \label{figure_framework}
\end{figure*}

\section{Background}

\textbf{Graph Neural Networks:}
In the past years, graph neural networks \cite{spectralgcn,patchysan,ChebNet,GCN,GIN,GAT} have attracted more and more attention. 
GNNs are applied to not only graph tasks \cite{CitationDatasets} (\textit{e.g.}, recommend systems \cite{AddedRef-RecommendationSystem}) 
but also other applications \cite{AddedRef-Distillation,AddedRef-Quantization} (\textit{e.g.}, computer vision \cite{GAT-CV}). 
In particular, 
graph convolution network (GCN) \cite{GCN} 
has become an important baseline. 
By introducing self-attention techniques 
\cite{attention}, graph attention 
networks (GAT) \cite{GAT,GAT-2} are proposed and applied to other applications \cite{GAT-APP,GAT-CV}. 
As \cite{oversmooth} claimed that GNNs suffer from the over-smoothing problem, 
GALA \cite{GALA} develops the graph sharpening and ResGCN \cite{ResGCN} attempts to 
designs a deeper architecture.
The theoretical works \cite{oversmooth,Exponential,BenefitsOfDepth} 
have different views towards the depth of GNNs. 
Some works \cite{oversmooth,Exponential} claimed that the expressive power of GNN decreases with the increase 
of layers, while the others argue that the assumptions in \cite{Exponential} 
may not hold and deeper GNNs have stronger power \cite{BenefitsOfDepth}. 
Moreover, some works \cite{WL-test-GNN,GIN} investigate the expressive capability by 
showing the connection between Weisfeiler-Lehman test \cite{WL-test} and GNNs. 
Nevertheless, \textit{most of them neglect the inefficiency problem of GNNs}.

\textbf{Efficient Graph Neural Networks:}
To accelerate the optimization through batch gradient descent to GNN without 
too much deviation, 
several models \cite{GraphSAGE,FastGCN,StoGCN,GraphSAINT} propose to sample data points according to graph topology. 
These models propose different sampling strategies to obtain stable results. 
GraphSAGE \cite{GraphSAGE} produces a subgraph with limited neighbors for 
each node while FastGCN \cite{FastGCN} samples fixed nodes for each layer 
with the importance sampling. 
The variance of sampling is further controlled in \cite{StoGCN}. 
Cluster-GCN \cite{Cluster-GCN} aims to generate an approximate graph 
with plenty of connected components so that each component can be used 
as a batch per step. 
AnchorGAE \cite{AnchorGAE} proposes to accelerate the graph operation by introducing the anchors 
to convert the original graph into a bipartite one 
so that the complexity can be reduced to $\mathcal{O}(n)$ compared with the existing models \cite{AdaGAE,ProjectedClustering-TIP}. 
SGC \cite{SGC} simplifies GCN by setting all activations of middle layers 
as linear functions and SSGC \cite{S2GC} further improves it. 
In summary, \textit{SGNN proposed in this paper retains the non-linearity and requires 
no node sampling or sub-graph sampling.} 
L2-GCN \cite{L2-GCN} attempts to extend the idea of the classical stacked auto-encoder 
into the popular GCN while DGL-GNN \cite{DGL-GNN} further develops a parallel version. 
They both fail to train all GNN modules jointly 
but \textit{SGNN firstly offers a novel framework to train them 
like training layers in a conventional neural network.}

\textbf{Connections to Existing Models:} 
Stacked Auto-Encoder (\textit{SAE}) \cite{SAE} is a model applied to the pre-training of neural networks. 
It trains the current two-layer auto-encoder \cite{AE} and then feeds the latent 
features output by the middle layer to the next auto-encoder. 
The model is often used as a pre-training model instead of a formal model. 
MGAE \cite{MGAE} is an extension of SAE and its fundamental module is graph auto-encoder \cite{GAE}. 
The main difference compared with the proposed model is whether each module 
could be perceived by modules from both forward and backward directions. 
The stack paradigm is similar to the classical boosting models \cite{AdaBoost,GBDT,XGBoost} while 
some works \cite{BoostingNN,EnsembleNN} also investigated the boosting algorithm of neural networks. 
In recent years, some boosting GNN models \cite{GBDTGNN,AdaGNN} are also developed. 
The most boosting algorithms (\textit{e.g.}, \cite{BoostingNN,GBDTGNN}) aim to learn a prediction function gradually 
while the proposed SGNN aims to learn ideal embeddings gradually. 
Note that AdaGCN \cite{AdaGNN} is also trained gradually and the features 
are combined using AdaBoost \cite{AdaBoost}. 
More importantly, \textit{all these boosting methods for GNNs are only 
trained forward and the backward training is missing}. 
Deep neural interface \cite{DNI} proposes to decouple neural networks to 
asynchronously accelerate the computation of gradients. \textit{The decoupling 
is an acceleration trick to compute the gradients of $L$-layer networks}, 
while SGNN proposed in this paper explicitly separates an $L$-layer GNN into 
$L$ simple modules. In other words, \textit{the ultimate goal of SGNN is not to optimize 
an $L$-layer GNN}.

\section{Proposed Method} \label{section_method}
Motivated by SAE \cite{SAE} and the fact that the simplified models \cite{SGC,S2GC}
are highly efficient for GNN, 
we therefore rethink the substantial difference between the stacked networks 
and multi-layer GNNs. 
To sum up, we attempt to answer the following two questions in this paper:
\begin{itemize}
    \item [Q1:] \textit{How to decouple a complex GNN into multiple simple GNNs and train them jointly?}
    \item [Q2:] \textit{How does the decoupling affect the representational capacity and final performance?}
\end{itemize}
We will discuss the first question in this section and then elaborate on another one in Section \ref{section_theo}.

\subsection{Preliminary}
Each decoupled GNN model of the proposed model is named as a module and 
the $t$-th module is denoted by $\mathcal{M}_t$ for simplicity. 
The vector and matrix are denoted by lower-case and upper-case letters in bold, respectively. 
$\|\cdot\|$ represents the Frobenius norm. 
Given a graph, let $\bm A \in \mathbb R^{n \times n}$ be adjacency matrix 
and $\bm X \in \mathbb R^{n \times d}$ be node features. 
A typical GNN layer can be usually defined as 
\begin{equation} \label{eq_graph_conv}
    \bm H = f(\bm A, \bm X, \bm W) = \varphi (\bm P \bm X \bm W), 
\end{equation}
where $\bm W$ is projection coefficient and 
$\bm P = \phi(\bm A)$ is a function of $\bm A$. 
When we discuss each individual module, we assume that $\bm W \in \mathbb{R}^{d \times k}$
for simplicity. 
For example, GCN \cite{GCN} defines $\bm P$ as
$\bm P_{\rm GCN} = \bm D^{-\frac{1}{2}} (\bm A + \bm I) \bm D^{-\frac{1}{2}}$ and 
$\bm D$ is the degree matrix of $\bm A + \bm I$. 
When multiple layers are integrated, the learned representation given by multiple GNN layers can be written as 
\begin{equation}
    \begin{split}
    \bm H & = f_1 \circ f_2 \circ \cdots \circ f_L (\bm A, \bm X, \bm W_1, \ldots, \bm W_L) \\
    & = \varphi_L(\bm P \varphi_{L-1}(\cdots \varphi_1(\bm P \bm X \bm W_1) \cdots) \bm W_L), 
    \end{split}
\end{equation}
where $L$ is the amount of layers. 
Assume that the average number of neighbors is $c$. 
To compute $\bm H$, each sample will need $\mathcal O(c^L)$ extra samples. 
If the depth is large and the graph is connected, 
then all nodes have to be engaged to compute for one node. 
The uncontrolled batch size results in the time-consuming training. 
In vanilla GNNs, the computational complexity is 
$\mathcal O(K L \|\bm A\|_0 \sum _{i=0}^{m-1} d_i d_{i+1})$ on sparse graphs 
where $K$ is the number of iterations and $d_i$ is the dimension of $\bm W_i$. 
For large-scale datasets, both time and space complexity are too expensive.

\subsection{Stacked Graph Neural Networks}
Although the stacked networks usually have more parameters than multi-layer networks, 
which frequently indicates that the stacked networks may be more powerful, 
they only serve as a technique for pre-training. 
Specifically speaking, 
they simply transfer the representations learned by the current network to the 
next one but \textit{no feedback is passed back}. 
It causes the invisibility of the succeeding modules and the final objective. 
As a result of the unreliability of the former modules, 
the stacked model is conventionally used as an unsupervised pre-training model. 

Rethinking the learning process of a network, 
multiple layers are optimized simultaneously by gradient-based methods 
where the gradient is calculated by the well-known backward propagation algorithm \cite{BP}. 
The algorithm consists of forward propagation (FP) and backward propagation (BP). 
FP computes the required values for BP, which can be viewed as an information 
delivery process. 
Note that FP is similar to the training of the stacked networks. 
Specifically, transferring the output of the current module to the next one in the stacked network 
is like the computation of neurons layer by layer during FP.
Inspired by this, we aim to design a BP-like training strategy, namely \textit{backward training (BT)}, 
so that the former modules could be tuned according to the feedback. 
The core idea of our stacked graph neural network (SGNN) is shown in Figure \ref{figure_framework}.

\subsubsection{Separability: Crucial Concept for Efficiency}
Before introducing SGNN in detail, we formally introduce the key concept 
and core motivation of how to accelerate GNN via SGNN. 
\begin{definition} \label{def_separable}
If a GNN model can be formulated as 
$$f(\bm A, \bm X, \bm W) = f_1(f_0 (\bm A, \bm X), \bm W),$$ 
then it is a separable GNN. 
If it can be further formulated as 
$$f(\bm A, \bm X, \bm W) = f_1(f_0 (\bm A, \bm X), \bm W) = g_1 (\bm A, g_0(\bm X, \bm W)),$$ 
then it is a fully-separable GNN. 
\end{definition}
To keep simplicity, define the set of separable GNNs as 
\begin{equation}
    \begin{aligned}
    \mathcal{F}_k = \{& f: \mathbb R^{n \times n} \times \mathbb R^{n \times d} \times \mathbb R^{d \times k} \mapsto \mathbb R^{n \times k} \\ 
    & | f(\bm A, \bm X, \bm W) = f_1(f_0 (\bm A, \bm X), \bm W) \}
    \end{aligned}
\end{equation}
and the set of fully-separable GNNs as 
\begin{equation}
    \begin{aligned}
        \mathcal{F}_k^* = \{& f: \mathbb R^{n \times n} \times \mathbb R^{n \times d} \times \mathbb R^{d \times k} \mapsto \mathbb R | f(\bm A, \bm X, \bm W) \\ 
        & = f_1(f_0 (\bm A, \bm X), \bm W) = g_1 (\bm A, g_0(\bm X, \bm W)) \}.
    \end{aligned}
\end{equation}
Note that most single-layer GNN models are separable. 
For instance, SGC \cite{SGC} is fully-separable where $f_0(\bm A, \bm X) = \bm P^m \bm X$ and $f_1(f_0(\bm A, \bm X), \bm W) = \varphi(f_0(\bm A, \bm X) \cdot \bm W)$, 
while the single-layer GIN \cite{GNN} is separable but not fully-separable since 
$\bm P \cdot {\rm MLP}(\bm X) \neq {\rm MLP}(\bm{P X})$ usually holds. 
JKNet\cite{JKNet} consisting of one layer is also separable but not fully-separable. 
However, a single-layer GAT \cite{GAT} is not separable since the graph 
operation is relevant to $\bm W$. 

\textit{The separable property actually factorizes a GNN model to 2 parts, graph operation $f_0$
and neural operation $f_1$}. 
Since all dependencies among nodes in GNNs are caused by the graph operation,
one can compute $\bm X' = f_0 (\bm A, \bm X)$ once (like \textbf{preprocessing}) in separable GNNs 
and then the GNN is converted into a typical network. 
After computing $\bm X'$, the information contained in graph has been 
passed into $\bm X'$ and the succeeding sampling would not affect 
the topology of graph. 
Therefore, we can obtain a highly efficient GNN model that can be optimized by SGD, 
provided that each module is separable. 
On the other hand, \textit{the fully-separable condition is essential for the backward training 
to pass back the information over multiple modules}. 
Since most single-layer GNNs are separable but not fully-separable, 
we show how to revise separable GNNs to \textbf{introduce the fully-separability}. 

Then, we formally clarify the core idea of SGNN by showing how to handle \textbf{\textit{Q1}}.

\subsubsection{Forward Training (FT)}
\textit{The first challenge} is how to set the training objective for each 
module $\mathcal{M}_t$. 
It is crucial to apply SGNN to both supervised and unsupervised scenes. 
Suppose that we have a separable GNN module $\mathcal{M}$ and
let $\bm H = f(\bm A, \bm X, \bm W)$ be the features learned by the separable GNN module. 
For the unsupervised cases, if $\mathcal{M}_t$ is a GAE, then the loss of FT is formulated as 
\begin{equation} \label{loss_GAE}
    \mathcal{L}_{FT} = \ell (\bm A, \bm X, \bm W) = d(\bm A, \kappa (\bm H)),
\end{equation}
where $d(\cdot, \cdot)$ represents the metric function and $\kappa: \mathbb{R}^{n \times k} \mapsto \mathbb{R}^{n \times n}$ 
is a mapping function. 
For instance, a simple loss introduced by \cite{GAE} is 
$d(\bm A, \kappa(\bm H)) = KL(\bm A \| \sigma(\bm H \bm H^T))$ where
$\sigma(\cdot)$ is the sigmoid function, 
and $KL(\cdot \| \cdot)$ is the Kullback-Leibler divergence. 
The other options include but not limited to symmetric content reconstruction \cite{GALA} and 
graph contrastive learning \cite{MultiViewContrastive}. 
For modules with supervision information, a projection matrix $\bm R \in \mathbb{R}^{k \times c}$ 
is introduced to map the $k$-dimension embedding vector into soft labels with $c$ classes. 
For the node classification, the loss can be simply set as 
\begin{equation} \label{loss_GCN}
    \mathcal{L}_{FT} = KL ( \bm Y \| \textrm{softmax}(\bm H \bm R)) ,
\end{equation}
where $\bm Y \in \mathbb{R}^{n \times c}$ is the supervision information for supervised tasks. 
Note that the above loss is equivalent to the classical softmax regression if $\bm H$ is constant. 
The loss could also be link prediction, graph classification, \textit{etc}. 
Although base modules can utilize diverse losses, we only discuss the situation 
that all modules use the same kind of loss in this paper for simplicity. 

        \begin{algorithm}[h]
            \IncMargin{1em}
            \caption{Procedure of Stacked Graph Neural Networks Composed of $L$ Modules}
            \label{alg_procedure}
            \SetKwInOut{Input}{\textit{Input}}\SetKwInOut{Output}{\textit{Output}}
            \Input{Adjacency matrix $\bm A$, feature matrix $\bm X$, balance coefficient $\eta$, the number of epochs $K$, $L$ separable GNN modules $\{\mathcal{M}_t\}_{t=1}^L$, $\bm H_0 \leftarrow \bm X$.} 
	        \Output{Features output by $\mathcal{M}_L$.}
            \For{$i = 1, \ldots, K$}{
            \textcolor{codeAnnotation}{\textit{\# Forward Train Stacked Graph Neural Networks}} \\
            \For{$t = 1, 2, \ldots, L-1$}{
                Feed the current features to $\mathcal{M}_t$ and reset $\bm U_t$: $\bm X_{t} \leftarrow \bm H_{t-1}$, $\bm U_t \leftarrow \bm I$. \\
                \textcolor{codeAnnotation}{\textit{\# Train with only $\mathcal{L}_{FT}$ at the first forward training}} \\
                ${\bm X_t}' = f_0^{(t)} (\bm A, \bm X_t)$ ~~ \textcolor{codeAnnotation}{\textit{\# Preprocessing for mini-batch algorithms}}\\
                Compute loss $\mathcal{L}^{(t)} \leftarrow \mathcal{L}_{FT}^{(t)}$ if $i==1$ else $\mathcal{L}_{FT}^{(t)} + \eta \mathcal{L}_{BT}^{(t)}$\\
                Train $\mathcal{M}_t$ by optimizing $\min_{\bm W_{t}} \mathcal L^{(t)}$ based on mini-batch algorithms. \\
                Obtain the features: $\bm H_t \leftarrow f^{(t)}_1 ({\bm X_t}', \bm W_t)$. 
            }
            $\bm X_L \leftarrow \bm H_t$, $\bm U_L \leftarrow \bm I$, ${\bm X}_L' = f_0^{(L)}(\bm A, \bm X_L)$. \\ 
            Train $\mathcal{M}_L$ by optimizing $\min_{\bm W_{L}, \bm U_L} \mathcal L_{FT}^{(L)}$ based on mini-batch algorithms. \\
            \textcolor{codeAnnotation}{\textit{\# Backward Train Stacked Graph Neural Networks}} \\
            \For{$t = L-1, L-2, \ldots,1$}{
                Compute the expected output feature of $\mathcal{M}_t$: $\bm Z_{t+1} \leftarrow (g^*_{0})_{t+1} (\bm X_{t+1}, \bm U_{t+1})$. \\
                Train $\mathcal{M}_t$ by optimizing $\min_{\bm W_t, \bm U_{t}} \mathcal{L}_{FT}^{(t)} + \eta \mathcal{L}_{BT}^{(t)} $ based on mini-batch algorithms. 
            }
            }
            \DecMargin{1em}
        \end{algorithm}

\subsubsection{Backward Training (BT)}
\textit{The second challenge} is how to train multiple separable GNNs simultaneously 
in order to ensure performance. 
Roughly speaking, the gradients of all layers in classical multi-layer neural networks are computed exactly due to
the repeated delivery of information by FP and BP. 
\textit{BP lets the shallow layers perceive the deep ones through the feedback}.
\textit{In SGNN, the tail modules are invisible to the head ones in FT.} 
We accordingly design the \textit{backward training} (\textbf{BT}) for SGNN. 
To achieve the reverse information delivery, 
the core idea is to \textbf{introduce the fully-separability} via defining \textbf{expected features}. 
For a separable GNN layer serving as a module in SGNN  
\begin{equation} \label{eq_original_separable_GNN}
    \bm H = f(\bm A, \bm X, \bm W) = f_1(f_0(\bm A, \bm X), \bm W), \forall f \in \mathcal{F}_k , 
\end{equation}
where $\bm W \in \mathbb R^{d \times k}$, 
we aim at tuning its input so that the previous module can be aware of 
what kind of representations are required by the current module. However, $\bm X$ is not a learnable parameter. 
A direct scheme is to use a transformation ($\mathbb{R}^d \mapsto \mathbb{R}^d$) to tune the input of the module 
\begin{equation}
    \bm H = f_1(f_0(\bm A, \psi(\bm X \bm U)), \bm W)
\end{equation}
where $\bm U \in \mathbb R^{d \times d}$ is a learnable square matrix 
and $\psi: \mathbb{R}^d \mapsto \mathbb{R}^d$ is a non-parametric function. 
Clearly, $f_0(\bm A, \bm X, \bm U)$ can be regarded as a parametric GNN. 
To retain the high efficiency owing to $f_0$, we therefore further constrain 
the revised $f_0 (\bm A, \psi(\bm X \bm U))$ as a fully-separable GNN, \textit{i.e.}, 
\begin{equation}
    \begin{aligned}
    f_0(\bm A, \bm X \bm U) = f^* (\bm A, \bm X, \bm U) & = f_1^*(f_0^*(\bm A, \bm X), \bm U) \\
    & = g_1^*(\bm A, g_0^*(\bm X, \bm U)), 
    \end{aligned}
\end{equation}
where $f^* \in \mathcal{F}_d^*$. 
Note that if $\bm U$ is fixed as $\bm I$ (or other constant matrices), 
then the modified layer is equivalent to original separable GNN in Eq. (\ref{eq_original_separable_GNN}). 
To sum up, a separable GNN layer $f$ is modified by introducing the fully-separability, which is shown as 
\begin{equation}
    \begin{split}
    & \forall f \in \mathcal{F}_k, \bm H = f(\bm A, \bm X, \bm W) = f_1(f_0(\bm A, \bm X), \bm W) \\ 
    \Rightarrow & \forall f^* \in \mathcal{F}_d^*, \bm H = F(\bm A, \bm X, \bm U, \bm W) = f_1 (f^{*} (\bm A, \bm X, \bm U), \bm W) ,
    \end{split}
\end{equation}
where $F$ represents a function from $\mathbb{R}^{n \times n} \times \mathbb{R}^{n \times d} \times \mathbb{R}^{d \times d} \times \mathbb{R}^{d \times k}$ 
to $\mathbb{R}^{n \times k}$. 
Denote $\bm Z = g_{0}^* (\bm X, \bm U)$ 
and $\bm Z$ is the \textit{expected features}. 
Specifically, let $\bm Z_{t}$ be the learned expected features 
during the backward training of $\mathcal{M}_t$ 
and it serves as the expected input of $\mathcal{M}_t$ from $\mathcal{M}_{t-1}$. 
In the forward training, the delivery of information is based on the learned features $\bm H_t$, 
and $\bm Z_t$ plays the similar role in the backward training. 
The loss of backward training attempts to shrink the difference between 
the output feature $\bm H_t$ of $\mathcal{M}_t$ and expected input $\bm Z_{t+1}$ of $\mathcal{M}_{t+1}$, 
\begin{equation}
    \mathcal L_{BT}^{(t)} = d(\bm H_t, \bm Z_{t+1}) = d(F_{t}(\bm A, \bm X_t, \bm U_t, \bm W_t), \bm Z_{t+1}). 
\end{equation}
Note that $\mathcal L_{BT}$ is only activated after the first forward training 
leading to the final loss of $\mathcal{M}_t$ as 
\begin{equation}
    \mathcal{L}^{(t)} = \mathcal{L}_{FT}^{(t)} + \eta \mathcal{L}_{BT}^{(t)}, 
\end{equation}
and it is updated during each backward training. 
The introduction of $\bm Z$ will not limit the application of stochastic 
optimization since the expected features can also be sampled at each iteration 
without restrictions. 
The procedure is summarized as Algorithm \ref{alg_procedure}. 

Remark that $\bm U_t$ remains as the identity matrix during FT. 
This setting leads to each forward computation across $L$ base modules being equivalent 
to a forward propagation $L$-layer GNN. 
In other words, an SGNN with $L$ modules can be regarded as a decomposition of 
an $L$-layer GNN. 
One may concern that why not to learn $\bm U_t$ and $\bm W_t$ together in FT. In this case, 
we prefer to use $\bm U_t$ only for learning the expected features of 
$\mathcal{M}_{t+1}$ and the capability improvement from the co-learning $\bm U_t$ in 
FT could also be implemented by $\bm W$, which is equivalent to use GIN \cite{GIN} as base modules. 

\begin{table*}
    \centering
    \renewcommand\arraystretch{1.1}
    \caption{Comparison of different efficient GNN models. 
    ``\# Activations'' represents the number of activation functions 
    that can be used for each model at most. Preprocessing complexity is the required computational complexity before 
    applying mini-batch gradient descent. For time complexity, we report the computational complexity per update. 
    For simplicity, we assume that the original features and hidden features are all $d$-dimension. 
    For sampling-based methods, $r$ represents the number of sampled nodes. 
    $\|\bm A_{\mathcal{B}}\|_0$ represents the average node number of 
    the partitioned sub-graphs. }
    \label{table_complexity}
    \begin{tabular}{l c c c c c c}
        \hline \toprule
         & GraphSAGE & FastGCN  & Cluster-GCN & SGC & S$^2$GC & SGNN\\
        \hline
        \hline
        Separability & \xmark & \xmark & \xmark & \checkmark & \checkmark & \checkmark \\
        Sampling &  \checkmark & \checkmark & \checkmark & \xmark & \xmark & \xmark \\
        \# Activations & $L$& $L$ & $L$ & 1 & 1 & $L$ \\
        Preprocessing complexity & 0 & $\mathcal{O}(\|\bm A\|_0)$ & $\mathcal{O}(n)$ & $\mathcal{O}(L \|\bm A\|_0 d)$ & $\mathcal{O}(L(L+1)\|\bm A\|_0 d+ n d)$ & $\mathcal{O}(L K \|\bm A\|_0 d)$ \\
        Time complexity per update & $\mathcal{O}(r^L m d^2)$ & $\mathcal{O}(r L m d^2)$ & $\mathcal{O}(\|\bm A_{\mathcal{B}}\|_0 L d^2)$ & $\mathcal{O}(m d^2)$ & $\mathcal{O}(m d^2)$ & $\mathcal{O}(m d^2 )$\\

        \bottomrule\hline 
    \end{tabular}
\end{table*}

\subsection{Complexity} \label{appendix_complexity}
As each base module is assumed as a separable GNN, both FT and BT of $\mathcal{M}_t$ can 
be divided into two steps, the preprocessing step for graph operation and the training step 
for parameters learning. 
Denote the output dimension of $\mathcal{M}_t$
as $d_t$ and the dimension of original content feature as $d_0 = d$.
The preprocessing to compute ${\bm X}_t' = f_0(\bm A, \bm X_t)$ requires 
$\mathcal{O}(\|\bm A\|_0 d_t)$ cost. 
Suppose that each module is trained $E$ iterations and the batch size is set as $m$. 
Then the computation cost of the training step is $\mathcal{O}(E m d_{t-1} d_t)$. 
Note that only the GNN mapping is considered and the computation of the loss 
is ignored. Overall, the computational complexity of an SGNN with $L$ modules 
is approximately 
$\mathcal{O}(K\|\bm A\|_0 \sum_{t=0}^{L-1} d_t + E m \sum_{t=1}^{L} d_{t-1} d_{t})$.
Remark that the graph is only used once during every epoch and no sampling 
is processed on the graph such that the graph structure is completely retained 
which is unavailable in the existing fast GNNs. 
The coefficient of $\|\bm A\|_0$ is only $K \sum_{t=1}^{L-1} d_t$. 
The space complexity is only $\mathcal{O}(\|\bm A\|_0 + m d_{t-1} d_t)$. 
Therefore, the growth of graph scale will not affect the efficiency of SGNN. 
Due to the efficiency, most experiments of SGNN can be conducted on a PC with an NVIDIA 1660 (6GB) and 16GB RAM. 
To better clarify the advantages of the proposed SGNN, 
we summarize the characteristics of diverse efficient GNNs in Table \ref{table_complexity}.

\begin{figure*}[t]
    \centering
    \subcaptionbox{Cora (clustering) \label{figure_time_cora} }{
        \includegraphics[width=0.22\linewidth]{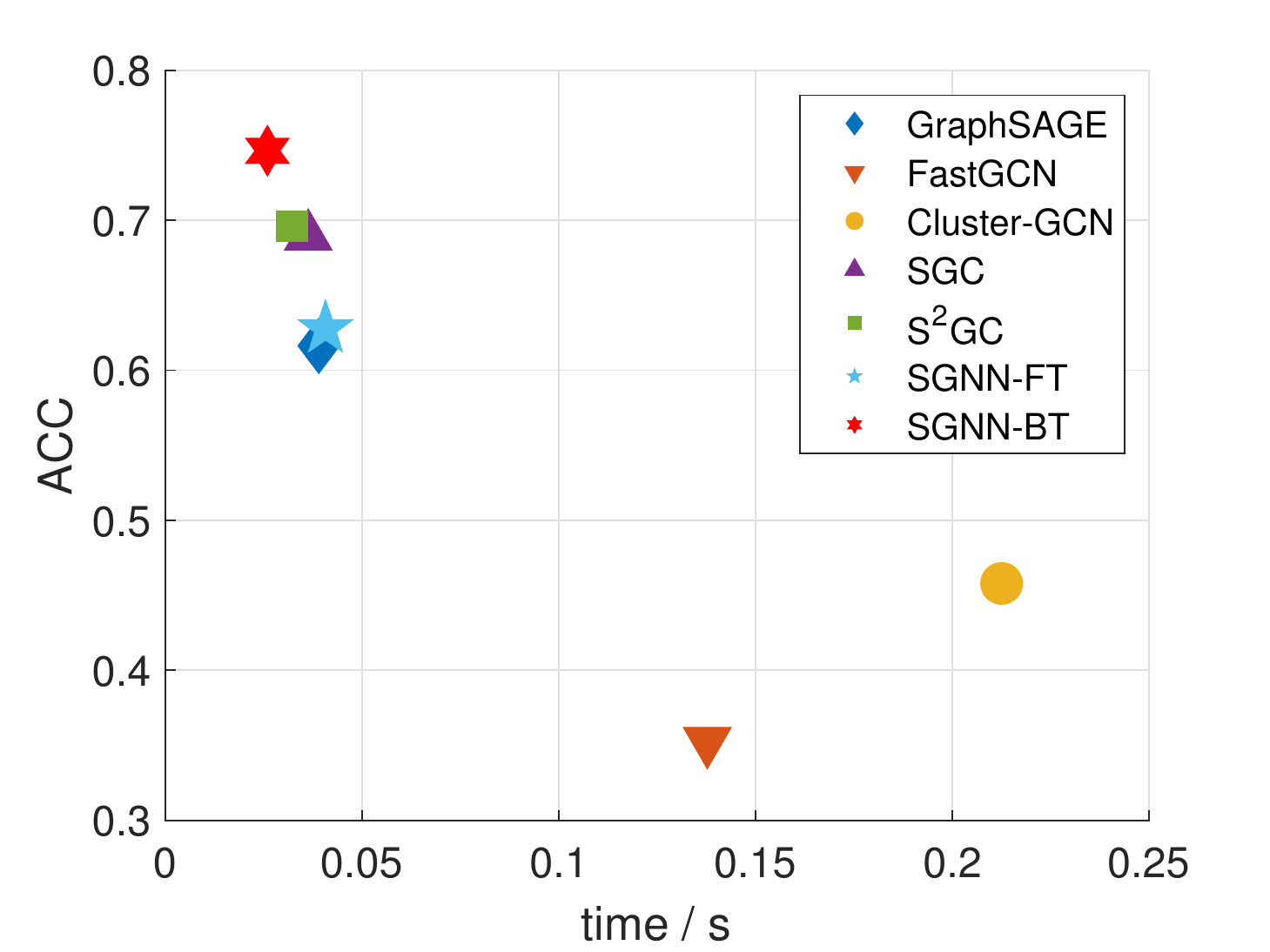}
    }
    \subcaptionbox{Citeseer (clustering) \label{figure_time_citeseer} }{
        \includegraphics[width=0.22\linewidth]{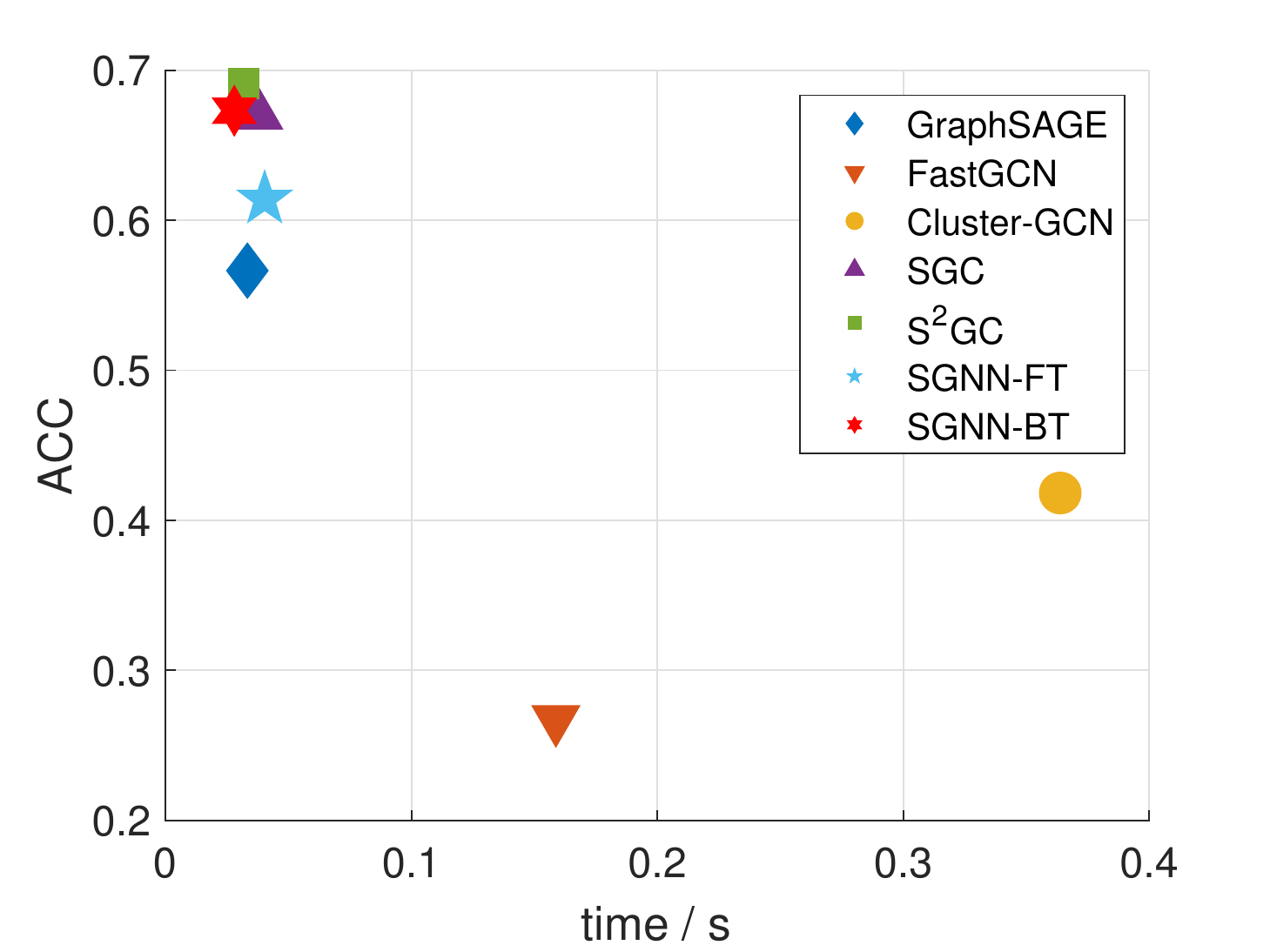}
    }
    \subcaptionbox{Pubmed (clustering) \label{figure_time_pubmed} }{
        \includegraphics[width=0.22\linewidth]{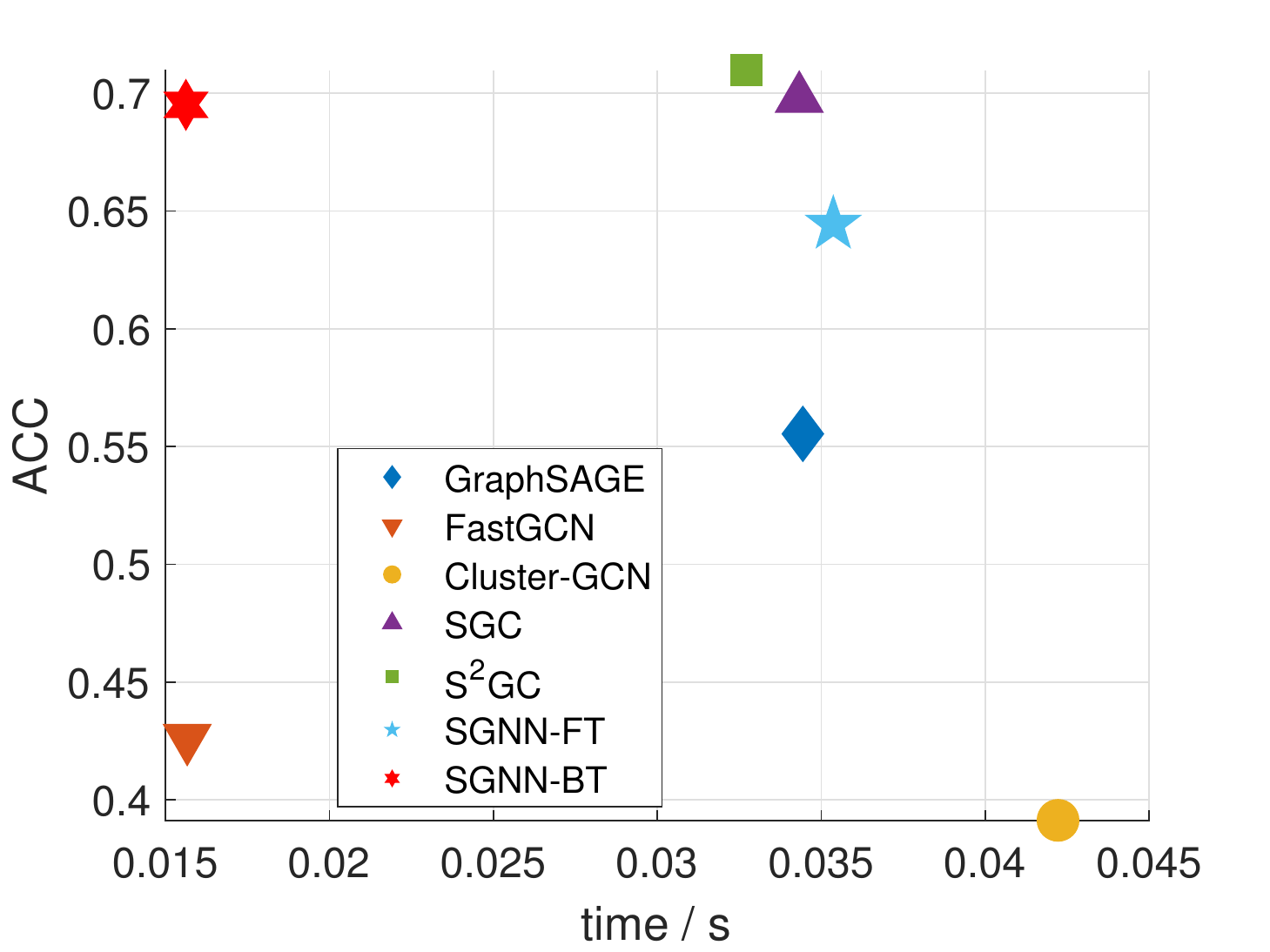}
    }
    \subcaptionbox{Reddit (clustering) \label{figure_time_reddit} }{
        \includegraphics[width=0.22\linewidth]{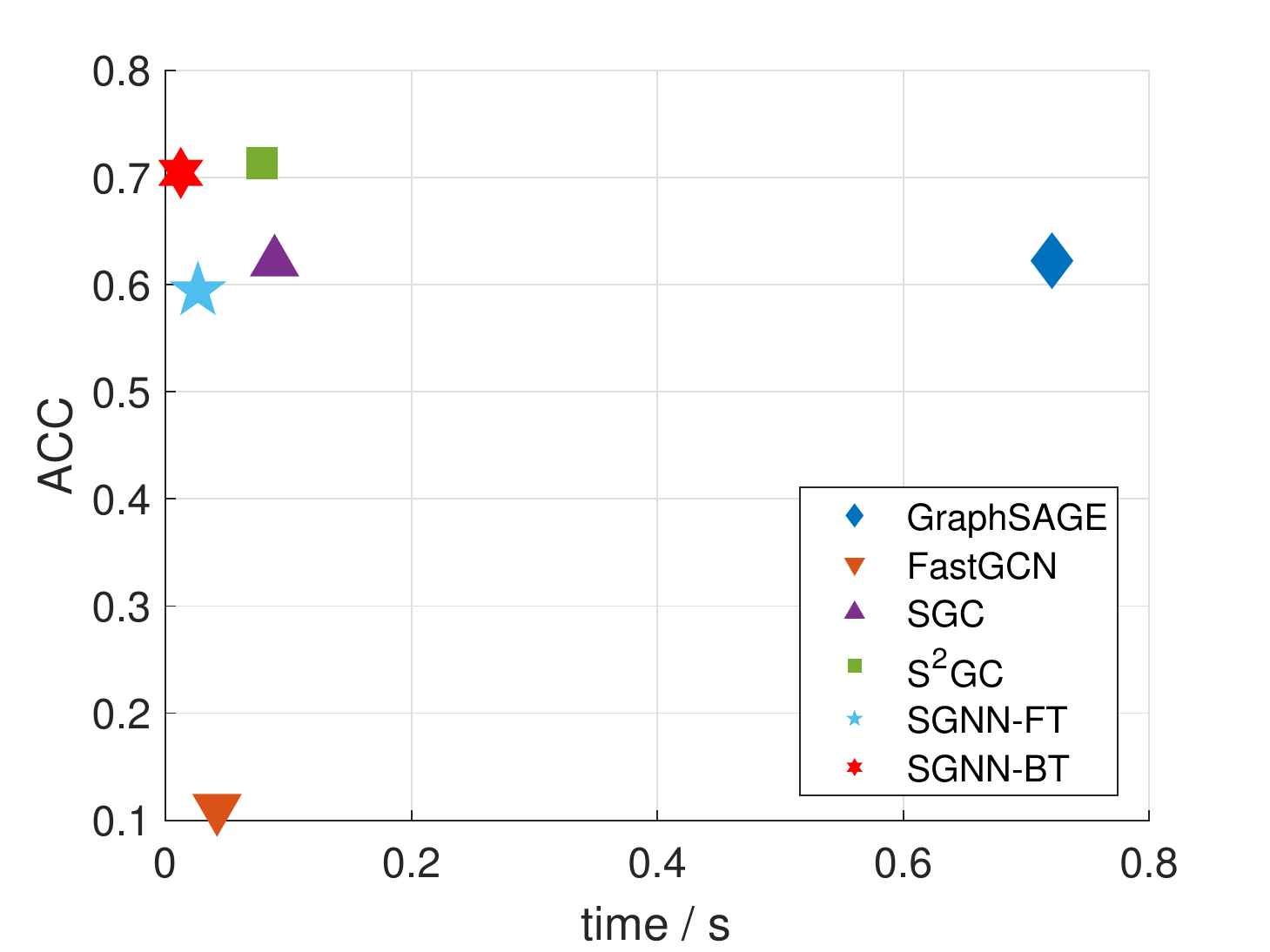}
    }

    \subcaptionbox{Cora (classification) \label{figure_time_cora_classification} }{
        \includegraphics[width=0.22\linewidth]{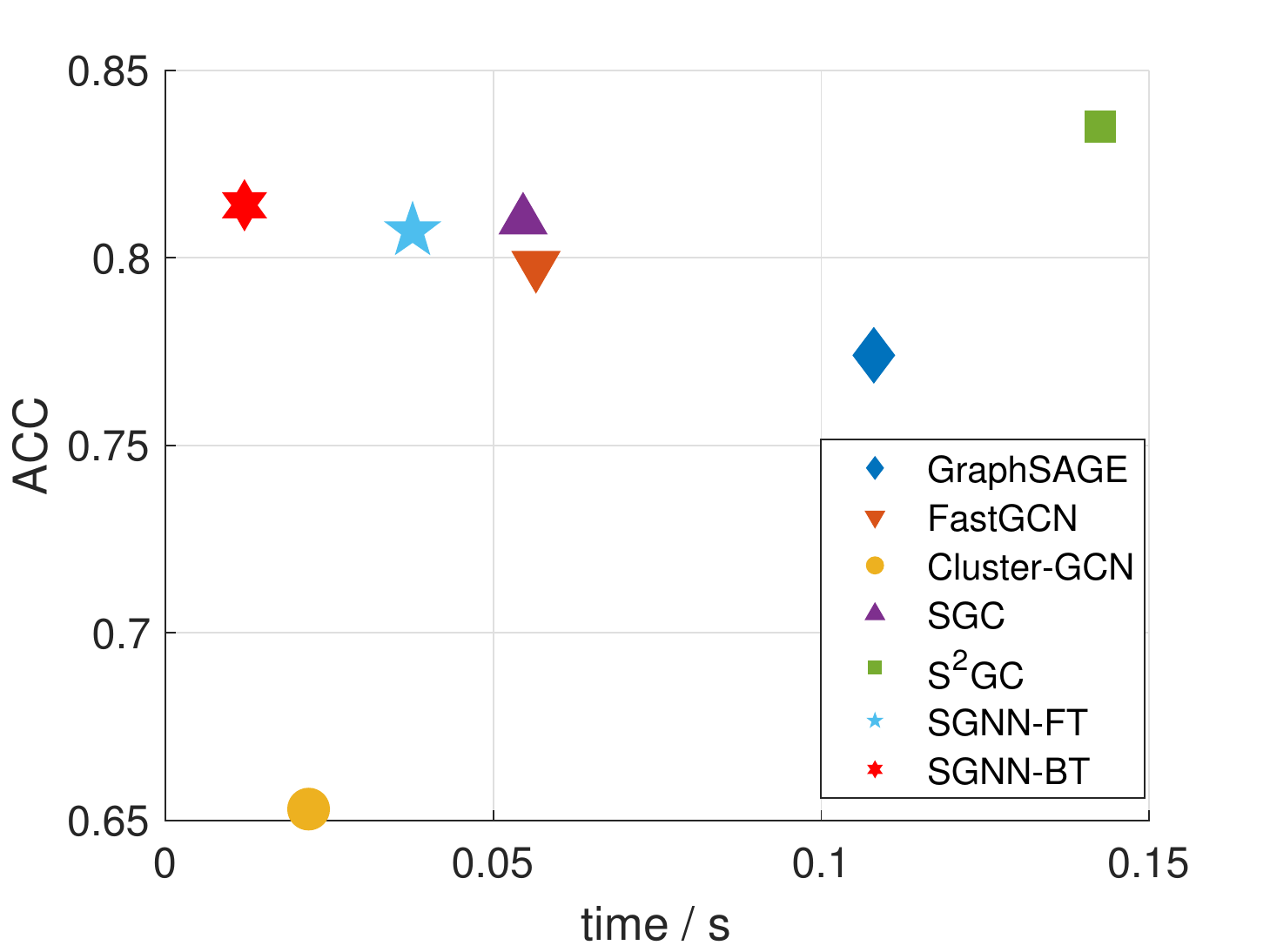}
    }
    \subcaptionbox{Citeseer (classification) \label{figure_time_citeseer_classification} }{
        \includegraphics[width=0.22\linewidth]{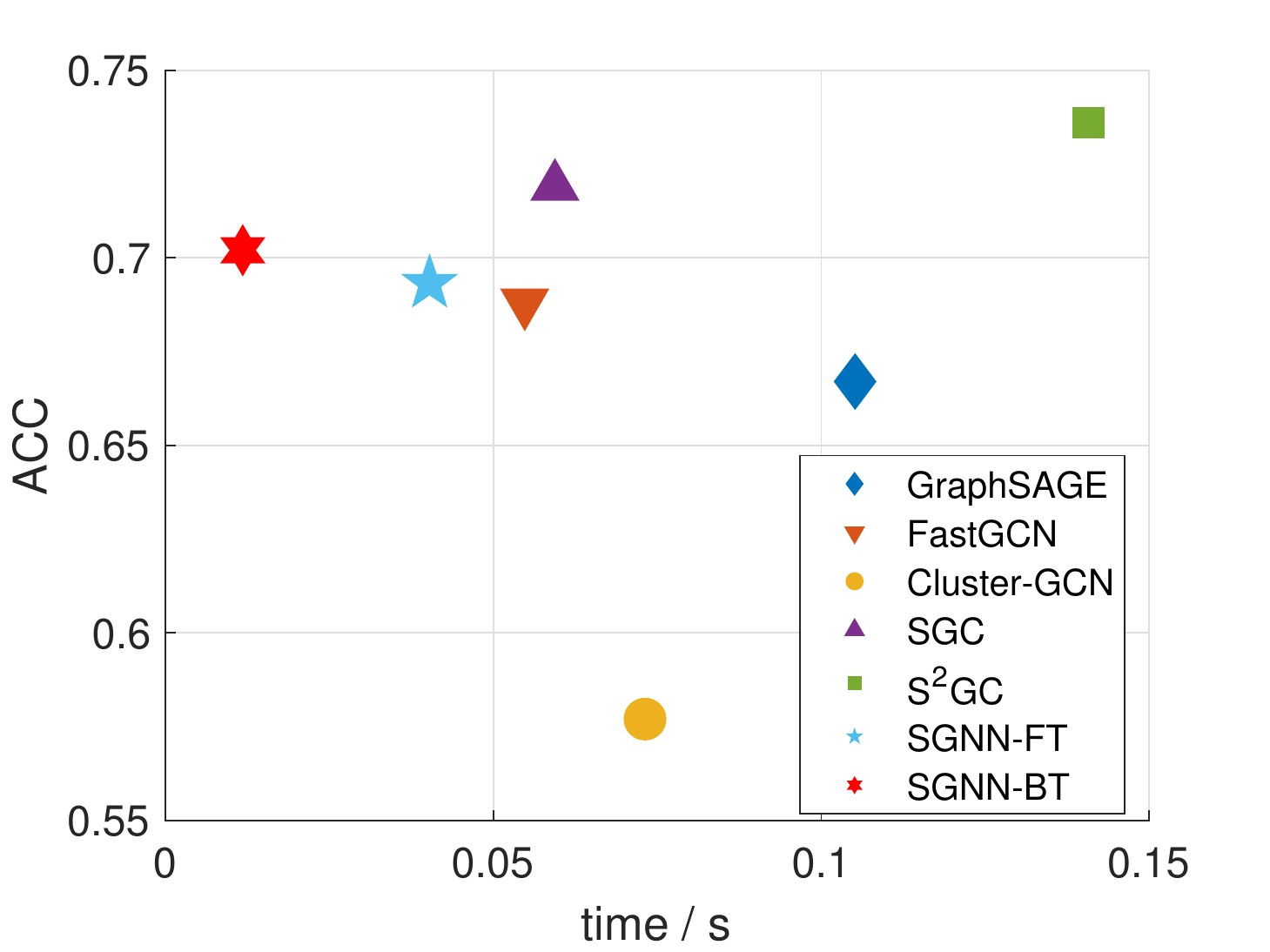}
    }
    \subcaptionbox{Pubmed (classification) \label{figure_time_pubmed_classification} }{
        \includegraphics[width=0.22\linewidth]{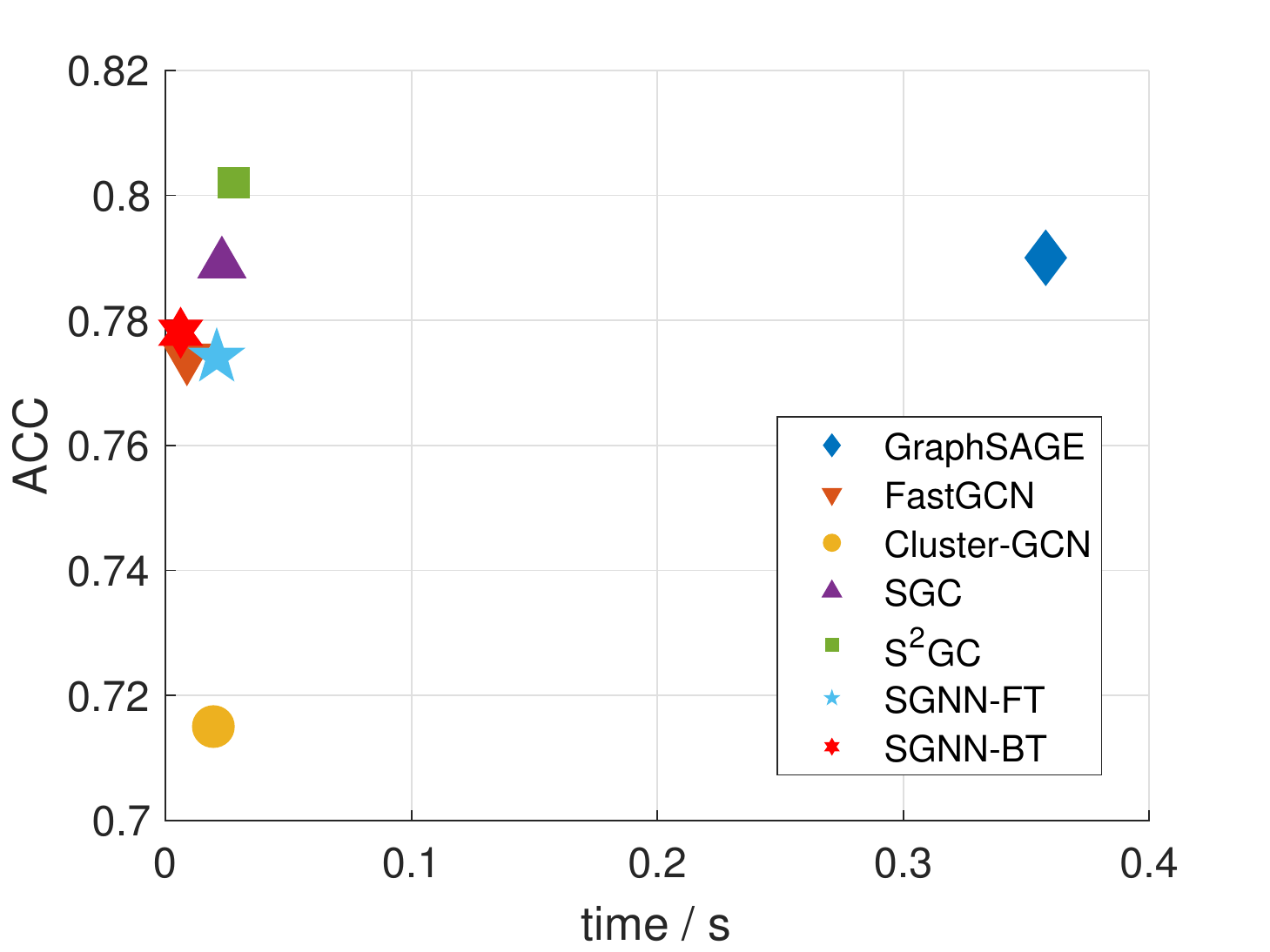}
    }
    \subcaptionbox{Reddit (classification) \label{figure_time_reddit_classification} }{
        \includegraphics[width=0.22\linewidth]{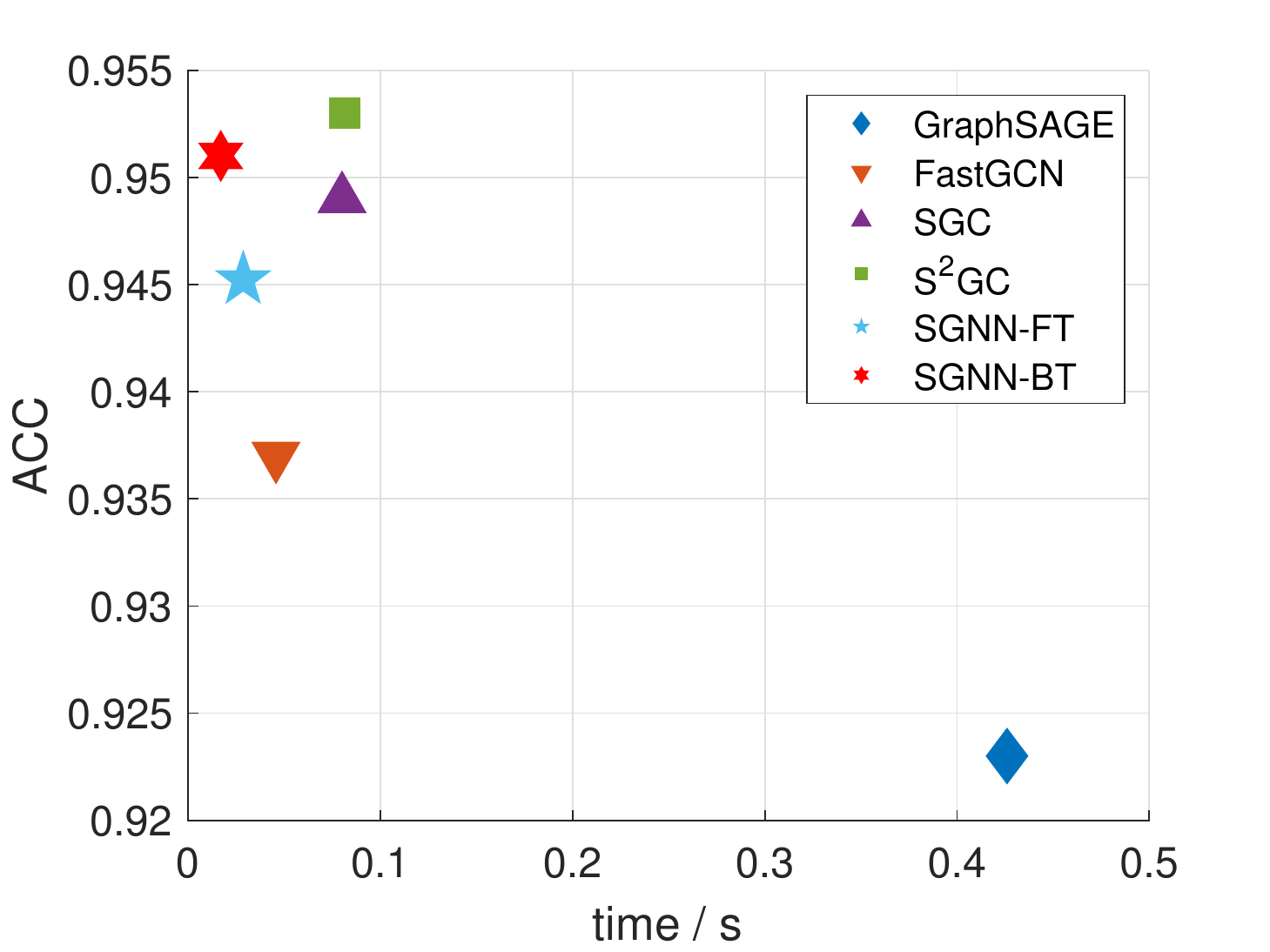}
    }
    \caption{Performance and training efficiency of several scalable GNNs. 
    The efficiency metric is computed by ``Consuming Time / \# Iterations''.  
    The consuming time begins from loading data into RAM. 
    The first line shows the result of node clustering 
    and the second line is the result of node classification. 
    Note that SGC is slower than SGNN since SGNN updates parameters more times 
    than SGC per graph preprocessing owing to the design of BT. 
    }
    \label{figure_time}
\end{figure*}

\section{Theoretical Analysis} \label{section_theo}

To answer \textbf{\textit{Q2}} raised in the beginning of Section \ref{section_method}, 
we discuss the impact of the decoupling in this section. 
Intuitively speaking, if an $L$-layer GNN achieves satisfactory results, 
then there exists $\{\bm W_t\}_{t=1}^L$ such that an SGNN with $L$ modules 
could achieve the same results. 
However, each $\mathcal{M}_t$ is trained greedily according to the forward 
training loss $\mathcal{L}_{FT}$, while middle layers of a multi-layer GNN 
are trained according to the same objective. 
The major concern is whether the embedding learned by a greedy 
strategy leads to an irreversible deviation in the forward training. 

In this section, we investigate the possible side effects 
on a specific SGNN comprised of 
unsupervised modules defined in Eq. (\ref{loss_GAE}) with linear activations. 
The conclusion is not apparent since simply setting $\bm W$ as 
an identity matrix does not prove it for GNN due to the existence of $\bm P$. 
Remark that a basic premise is that the previous module has achieved a reasonable result. 

Given a linear separable-GNN module $\mathcal{M}$ which is defined as $\bm H = \bm{P X W} \in \mathbb{R}^{n \times k}$, 
suppose that the forward training uses the reconstruction error $\|\bm P - \bm H \bm H^T\|$ as $\mathcal{L}_{FT}$. 
We first introduce the matrix angle to better understand whether the preconditions 
of Theorem \ref{theo_GAE} are practicable.

\begin{definition}
    Given two matrices $\bm B_1, \bm B_2 \in \mathbb{R}^{n \times n}$, we define the matrix angle as 
    $
    \theta(\bm B_1, \bm B_2) = \arccos \langle \bm B_1, \bm B_2 \rangle / (\|\bm B_1\| \cdot \|\bm B_2\|)
    $. 
\end{definition}
Before elaborating on theorems, we introduce the following assumption, 
which separates the discussions into two cases. 
\begin{assumption} \label{assumption_commute}
    $\bm X \bm X^T$ does not share the same eigenspace with $\bm P - \bm X \bm X^T$. 
\end{assumption}
Note that the above assumption is weak and frequently holds in most cases. 
For simplicity, $\bm U_r \in \mathbb R^{n \times r}$ ($r \in \mathbb{N}_+$) is the eigenvectors associated with $r$ leading eigenvalues. 
Under this assumption, 
we find that the error of $\mathcal{M}$ is upper-bounded by $\varepsilon$. 
\begin{theo} \label{theo_GAE}
    Let $\delta = 1- \cos(\theta_*/2)$ and $\theta_* = \theta ((\bm P - \bm X \bm X^T) \bm Q, \bm Q (\bm P - \bm X \bm X^T))$ 
    where $\bm Q = \bm U_o \bm U_o^T - \bm I / 2$ and $o = \min({\rm rank}(\bm X), k)$. 
    Under Assumption \ref{assumption_commute}, 
    if $\|\bm P - \bm H \bm H^T\| = \varepsilon \leq \mathcal{O}(\delta)$ and $\sigma_* \leq \mathcal{O}(\sqrt{\delta})$ 
    where $\sigma_*$ is the $(o+1)$-th largest singular value of $\bm X$, 
    then there exists $\bm W \in \mathbb{R}^{d \times k}$
    so that $\|\bm P - \bm H \bm H^T\| \leq \varepsilon$. 
    In other words, if $\varepsilon$ is small enough, then $\bm H \bm H^T$ could be 
    a better approximation than $\bm X \bm X^T$.
\end{theo}
Specially, if ${\rm rank}(\bm X) \leq k$ or $k=d$, $\sigma_* = 0$ so that $\sigma_* \leq \mathcal{O}(\sqrt{\delta})$ holds. 
From the above theorem, we claim that 
\textit{the error through $\mathcal{M}$ will not accumulate (i.e., bound by $\varepsilon$)} 
provided that the input $\bm X$, the output of the previous modules, is well-trained. 
We further provide an upper-bound of error if Assumption \ref{assumption_commute} 
does not hold. 
The following theorem shows the increasing speed of error is at most linear 
with the tail singular-values. 
\begin{theo} \label{theo_assumption_fail}
    If Assumption \ref{assumption_commute} does not hold, 
    then there exists $\bm W \in \mathbb{R}^{d \times k}$ so that $\|\bm P - \bm H \bm H^T\| \leq \varepsilon + \mathcal{O}(\sigma_*^2)$. 
\end{theo}
\begin{corollary}
    Given an SGNN with $L$ linear modules $\{\mathcal{M}_t\}_{t=1}^L$ with $\mathcal L_{FT}^{(t)} = \|\bm P - \bm H_t \bm H_t^T\|$, 
    if $\bm P - \bm H_1 \bm H_1^T$ and $\bm H_1 \bm H_1^T$ share the same eigenspace, 
    then $\mathcal{L}^{(L)}_{FT} \leq \mathcal{L}^{(1)}_{FT} + \sum_{t=1}^{L-1} \mathcal{O}(\sigma_*^2(\bm H_{t}))$. 
\end{corollary}

Based on Theorem \ref{theo_assumption_fail}, we conclude that the residual would not accumulate rapidly 
when Assumption \ref{assumption_commute} does not hold. 
All proofs are put in appendix.

\begin{table}[b]
    \centering
    \renewcommand\arraystretch{1.1}
    \setlength{\tabcolsep}{1.5mm}
    \caption{Data Information}
    \label{table_datasets}
    \begin{tabular}{l c c c c c }
        \hline

        \hline
        Dataset & Nodes  & Edges & Classes & Features & Train / Val / Test Nodes  \\
        \hline
        \hline
        Cora & 2,708 & 5,429 & 7 & 1,433 & 140 / 500 / 1,000 \\
        Citeseer & 3,327 & 4,732 & 6 & 3,703 & 120 / 500 / 1,000 \\
        Pubmed & 19,717 & 44,338 & 3 & 500 & 60 / 500 / 1,000 \\
        Reddit & 233K & 11.6M & 41 & 602 & 152K / 24K / 55K \\
        \hline

        \hline
    \end{tabular}
\end{table}

\begin{table*}[t]
    \centering
    \large
    \setlength{\tabcolsep}{2.5mm}
    \caption{Node clustering results on 4 datasets: The boldface and underline results indicate the optimal and sub-optimal results respectively. 
    ``N/A'' means the Out-Of-Memory (OOM) exception. }
    \label{table_clustering}
    \begin{tabular}{l c c c c c c c c}
        \hline
        
        \hline
        \multirow{2}{*}{Datasets} & \multicolumn{2}{|c|}{Cora} & \multicolumn{2}{c|}{Citeseer} & \multicolumn{2}{c|}{PubMed} & \multicolumn{2}{c}{Reddit}\\
        & \multicolumn{1}{|c}{ACC} & NMI & \multicolumn{1}{|c}{ACC} & NMI & \multicolumn{1}{|c}{ACC} & NMI & \multicolumn{1}{|c}{ACC} & NMI \\
        \hline
        \hline
        K-Means & 0.4922 & 0.3210 & 0.5401 & 0.3054 & 0.5952 & 0.2780 & 0.1927 & 0.2349 \\
        ARGA & 0.6400 & 0.4490 & 0.5730 & 0.3500 & 0.6807 & 0.2757 & N/A & N/A \\
        MGAE & 0.6806 & 0.4892 & 0.6691 & 0.4158 & 0.5932 & 0.2957 & N/A & N/A \\
        GraphSAGE & 0.6163 & 0.4826 & 0.5664 & 0.3425 & 0.5554 & 0.0943 & 0.6225 & 0.7291 \\
        FastGAE & 0.3527 & 0.1553 & 0.2672 & 0.1178 & 0.4262 & 0.0442 & 0.1115 & 0.0715 \\
        ClusterGAE & 0.4579 & 0.2261 & 0.4182 & 0.1767 & 0.3913 & 0.0001 & N/A & N/A \\
        \hline
        GAE & 0.5960 & 0.4290 & 0.4080 & 0.1760 & 0.6861 & 0.2957 & N/A & N/A\\
        AGC (SGC) & 0.6892 & 0.5368 & 0.6700 & 0.4113 & 0.6978 & 0.3159 & 0.5833 & 0.6894 \\
        SGNN-FT & 0.6278 & 0.5075 & 0.6141 & 0.3776 & 0.6444 & 0.2312 & 0.5943 & 0.7156 \\
        SGNN-BT & \textbf{0.7463} & \textbf{0.5546} & \underline{0.6730} & \underline{0.4159} & 0.6951 & \textbf{0.3337} & \textbf{0.7042} & \textbf{0.7601} \\
        \hline
        S$^2$GC & 0.6960 & \underline{0.5471} & \textbf{0.6911} & \textbf{0.4287} & \textbf{0.7098} & \underline{0.3321} & 0.7011 & 0.7509 \\
        GAE-S$^2$GC & 0.6976 & 0.5317 & 0.6435 & 0.3969 & 0.6528 & 0.2452 & 0.6272 & 0.7158 \\
        SGNN-S$^2$GC & \underline{0.7223} & 0.5404 & \underline{0.6822} & \underline{0.4243} & \underline{0.7084} & 0.3302 & \underline{0.7023} & \underline{0.7575} \\
        \hline

        \hline
        
    \end{tabular}
\end{table*}

\begin{table*}[h]
    \centering
    \large
    \setlength{\tabcolsep}{2.5mm}
    \captionof{table}{Investigation about the impact of $L$ on SGNN and GAE regarding node clustering.}
    \label{table_depth}
    \begin{tabular}{c | c  c c c c c c c c c c}
        \hline
        
        \hline
        \multirow{3}{*}{$L$} & \multicolumn{3}{c}{Cora} & \multicolumn{3}{|c|}{Citeseer} & \multicolumn{3}{c}{Pubmed}\\
        & SGNN & GAE & FastGAE & \multicolumn{1}{|c}{SGNN} & GAE & \multicolumn{1}{c|}{FastGAE} & SGNN & GAE & FastGAE \\

        \hline
        \hline

        2 & 0.75 & 0.60 & 0.35 & 0.67 & 0.41 & 0.27 & 0.70 & 0.69 & 0.43 \\
        3 & 0.66 & 0.63 & 0.33 & 0.65 & 0.58 & 0.25 & 0.64 & 0.64 & 0.42 \\
        4 & 0.68 & 0.65 & 0.33 & 0.59 & 0.58 & 0.24 & 0.64 & 0.60 & 0.41 \\
        5 & 0.69 & 0.62 & 0.33 & 0.53 & 0.45 & 0.24 & 0.64 & 0.60 & 0.41 \\
        6 & 0.69 & 0.53 & 0.32 & 0.44 & 0.32 & 0.24 & 0.64 & 0.48 & 0.41 \\
        7 & 0.68 & 0.52 & 0.32 & 0.44 & 0.31 & 0.24 & 0.64 & 0.46 & N/A \\
        \hline

        \hline
        
    \end{tabular}
\end{table*}

\section{Experiments}
In this section, we conduct experiments to investigate: (1) \textit{whether the performance 
of SGNN could approach the performance of the original $L$-layer GNN in a highly-efficient way}; 
(2) \textit{what the impact of the non-linearity and flexibility brought by the decoupling is}. 

To sufficiently answer the above 2 problems, 
both node clustering and semi-supervised node classification are used. 
It should be emphasized that SGNN can be used for both transductive and inductive tasks. 
We mainly conduct experiments on datasets for transductive learning since 
the inefficiency problem is more likely to appear in transductive learning scenes. 
There are 4 common datasets, 
including Cora, Citeseer, Pubmed \cite{CitationDatasets}, and Reddit \cite{GraphSAGE}, 
used for our experiments. 
Cora and Citeseer contain thousands of nodes. 
Pubmed has nearly 20 thousands nodes and Reddit contains more than 200 thousands nodes, 
which are middle-scale and large-scale datasets respectively.
The details of four common datasets are shown in Table \ref{table_datasets}. 
Experiments on two OGB datasets can be found in Section \ref{section_OGB}.

\subsection{Node Clustering}

\subsubsection{Experimental Settings}
We first testify the effectiveness of SGNN on the node clustering. 
We compare our method against 10 methods, including a baseline clustering model K-means, 
three GCN models without considering training efficiency (GAE \cite{GAE}, ARGA \cite{AGAE}, MGAE \cite{MGAE}), 
and six fast GCN models with GAE-loss (GraphSAGE \cite{GraphSAGE}, FastGAE \cite{FastGCN}, ClusterGAE \cite{Cluster-GCN}, 
AGC \cite{AGC} (an unsupervised extension of SGC \cite{SGC}), S$^2$GC \cite{S2GC}, and GAE-S$^2$GC). 

The used codes are based on the publicly released implementation. 
To ensure fairness, all multi-layer GNN models consist of two layers and 
SGNN is comprised of two modules. 
For models that could be trained by stochastic algorithms, the size of mini-batch 
is set as 128. 
The learning rate is set as 0.001 and the number of epochs is set as 100. 
We set the size first layer as 128 and the second layer size as 64. 
$\bm U$ is initialized as an identity matrix, and $\eta$ is set as $10^3$ by default. 
The number of backward training is set as 5 or 10. 
To investigate the effectiveness of backward training, 
we report the experimental results with sufficient training for each module, 
which is denoted by SGNN-FT, while SGNN with the proposed backward training 
is represented by SGNN-BT. 

To study the performance of SGNN with different base models, we choose 
a fully-separable GNN, S$^2$GC, as the base model and this method is marked as 
SGNN-S$^2$GC. 
Note that SGNN-S$^2$GC is also trained with the proposed BT strategy, which is same as SGNN-BT. 
The original S$^2$GC does not apply any activation functions while 
we add a activation functions, which are same as the SGNN-FT, to the embeddings of each S$^2$GC base models. 
As S$^2$GC does not use the GAE framework, we also added a competitor, namely GAE-S$^2$GC, 
which uses S$^2$GC as the encoder, to ensure fairness.
For SGNN-S$^2$GC and GAE-S$^2$GC, all extra hyper-parameters are simply set according 
to the original paper of S$^2$GC for both node clustering and node classification. 
We do not tune the hyper-parameters of S$^2$GC manually. 

All methods are run five times and the average scores are recorded. 
The result are summarized in Table \ref{table_clustering}.

\begin{table*}[t]
    \begin{minipage}[t]{0.6\textwidth}
        \centering
        \large
        \renewcommand\arraystretch{1.0}
        \setlength{\tabcolsep}{2.5mm}
        \caption{Test accuracy (\%) of supervised SGNN averaged over 10 runs on 3 citation datasets}
        \label{table_classification_small}
        \begin{tabular}{l c c c c c c c c}
            \hline
            
            \hline
            \multirow{1}{*}{Datasets} & \multicolumn{1}{c}{Cora} & \multicolumn{1}{c}{Citeseer} & \multicolumn{1}{c}{Pubmed} \\
            \hline
            \hline
            GAT & 83.0 $\pm$ 0.7 & 72.5 $\pm$ 0.7 & 79.0 $\pm$ 0.3 \\
            DGI & 82.3 $\pm$ 0.6 & 71.8 $\pm$ 0.7 & 76.8 $\pm$ 0.6 \\
            GCNII & \textbf{85.5 $\pm$ 0.5} & \underline{73.4 $\pm$ 0.6} & \textbf{80.2 $\pm$ 0.4} \\
            FastGCN$^\dag$ & 79.8 $\pm$ 0.3 & 68.8 $\pm$ 0.6 & 77.4 $\pm$ 0.3 \\
            GraphSAGE$^\dag$ & 77.4 $\pm$ 0.6 & 66.7 $\pm$ 0.3 & 79.0 $\pm$ 0.3 \\
            Cluster-GCN$^\dag$ & 65.3 $\pm$ 3.9 & 57.7 $\pm$ 2.3 & 71.5 $\pm$ 1.9 \\
            \hline
            GCN & 81.4 $\pm$ 0.4 & 70.9 $\pm$ 0.5 & 79.0 $\pm$ 0.4 \\
            SGC & 81.0 $\pm$ 0.0 & 71.9 $\pm$ 0.1 & 78.9 $\pm$ 0.0 \\
            SGNN-FT & 81.0 $\pm$ 0.6 & 69.3 $\pm$ 0.4 & 77.4 $\pm$ 0.3 \\
            SGNN-BT & 81.4 $\pm$ 0.5 & 70.2 $\pm$ 0.4 & 77.8 $\pm$ 0.3 \\
            \hline
            S$^2$GC & 83.5 $\pm$ 0.0 & \textbf{73.6 $\pm$ 0.1} & \textbf{80.2 $\pm$ 0.0} & \\
            SGNN-S$^2$GC & \underline{83.8 $\pm$ 0.2} & 73.2 $\pm$ 0.3 & \textbf{80.2 $\pm$ 0.4} \\ 
            \hline

            \hline

        \end{tabular}
    \end{minipage}
    \hspace{1mm}
    \begin{minipage}[t]{0.38\textwidth}
        \large
        

        \centering
        \renewcommand\arraystretch{1.0}
        \setlength{\tabcolsep}{5mm}
        \caption{Test accuracy (\%) averaged over 5 runs on Reddit}
        \label{table_classification_reddit}
        \begin{tabular}{l c c }
            \hline
            
            \hline
            \multirow{1}{*}{Methods}  & \multicolumn{1}{c}{Reddit}\\
            \hline
            \hline
            GAT & N/A \\
            DGI & 94.0 \\
            SAGE-mean & 95.0 \\
            SAGE-GCN & 93.0 \\
            FastGCN & 93.7 \\
            Cluster-GCN & \textbf{96.6} \\
            \hline
            GCN & N/A \\
            SGC & 94.9 \\
            SGNN-FT & 94.54 $\pm$ 0.01 \\
            SGNN-BT & 95.10 $\pm$ 0.02 \\
            \hline 
            S$^2$GC & \underline{95.3} \\
            SGNN-S$^2$GC & 95.28 $\pm$ 0.03 \\
            \hline
    
            \hline
            
        \end{tabular}
    \end{minipage}
\end{table*}

\begin{figure*}[t]
    \small
    \centering
    \setlength{\fboxrule}{0.1pt}
    \setlength{\fboxsep}{1pt}
    \subcaptionbox{GCN on Cora}{
        \includegraphics[width=0.23\linewidth]{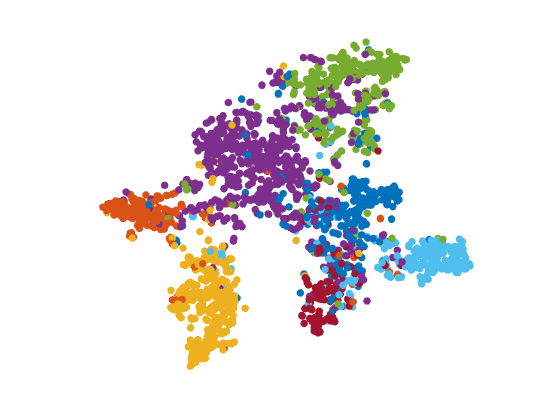}
    }
    \subcaptionbox{SGNN on Cora}{
        \includegraphics[width=0.23\linewidth]{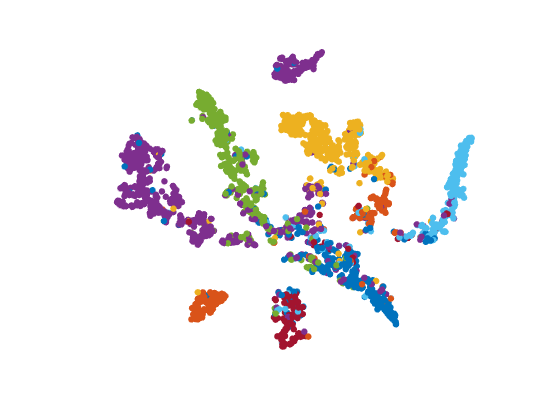}
    }
    \subcaptionbox{GCN on Citeseer}{
        \includegraphics[width=0.23\linewidth]{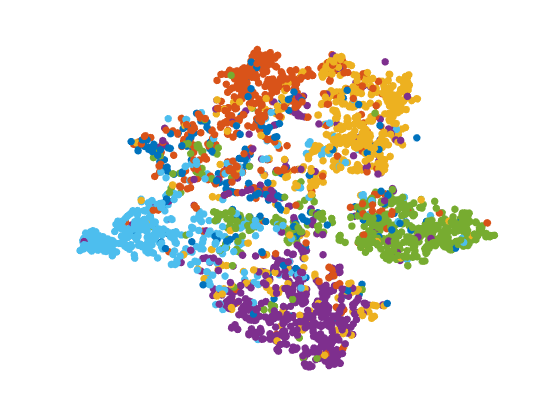}
    }
    \subcaptionbox{SGNN on Citeseer}{
        \includegraphics[width=0.23\linewidth]{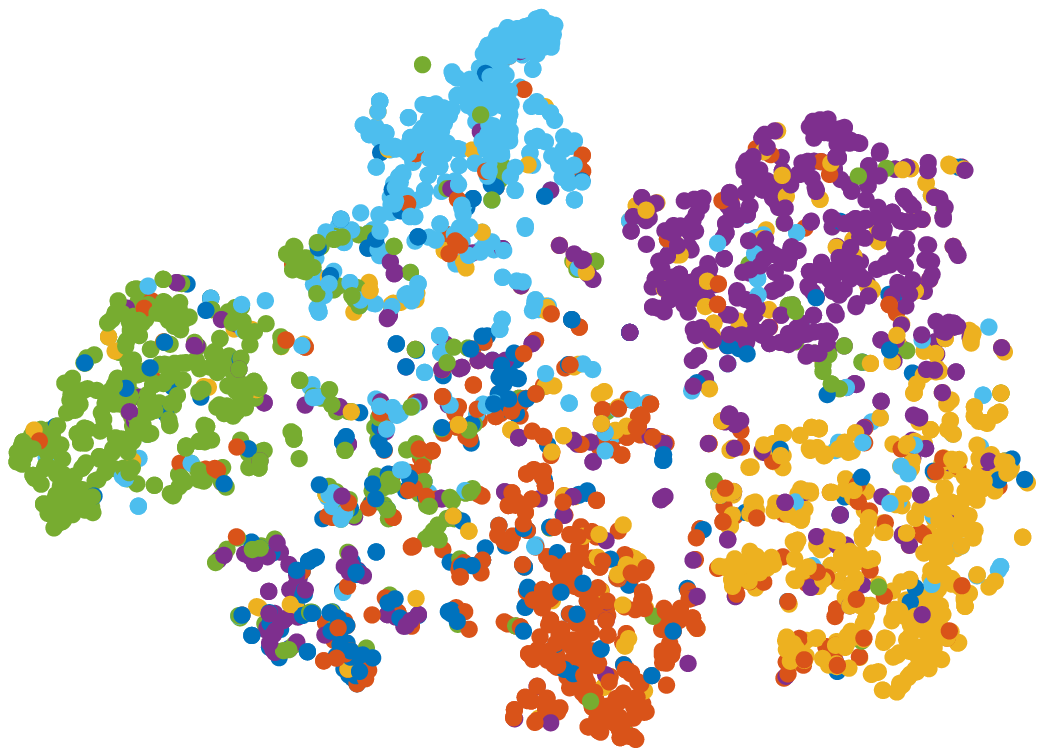}
    }
    
    \caption{Visualization of SGNN comprised of 3 modules and 3-layer GCN on node classification. 
    For SGNN, the output of $\mathcal{M}_3$ is visualized. For GCN, the output of the final GCN-layer is visualized. }
    \label{figure_visualization}
\end{figure*}

\begin{figure*}[t]
    \centering
    \setlength{\fboxrule}{0.1pt}
    \setlength{\fboxsep}{1pt}
    \subcaptionbox{$\bm H_0$: Input of $\mathcal{M}_1$}{
        \includegraphics[width=0.23\linewidth]{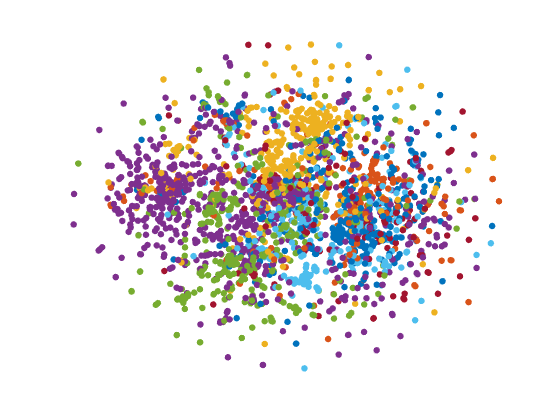}
    }
    \subcaptionbox{$\bm H_1$: From $\mathcal{M}_1$ to $\mathcal{M}_2$}{
        \includegraphics[width=0.23\linewidth]{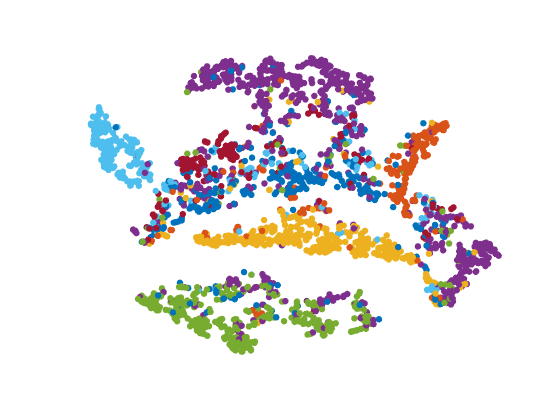}
    }
    \subcaptionbox{$\bm H_2$: From $\mathcal{M}_2$ to $\mathcal{M}_3$}{
        \includegraphics[width=0.23\linewidth]{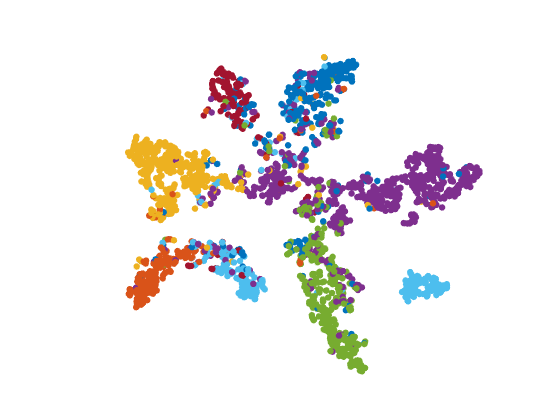}
    }
    \subcaptionbox{$\bm{H}_3$: Output of $\mathcal{M}_3$}{
        \includegraphics[width=0.23\linewidth]{figures/cora-X3.png}
    }

    \subcaptionbox{$\bm H_0$: Input of $\mathcal{M}_1$}{
        \includegraphics[width=0.23\linewidth]{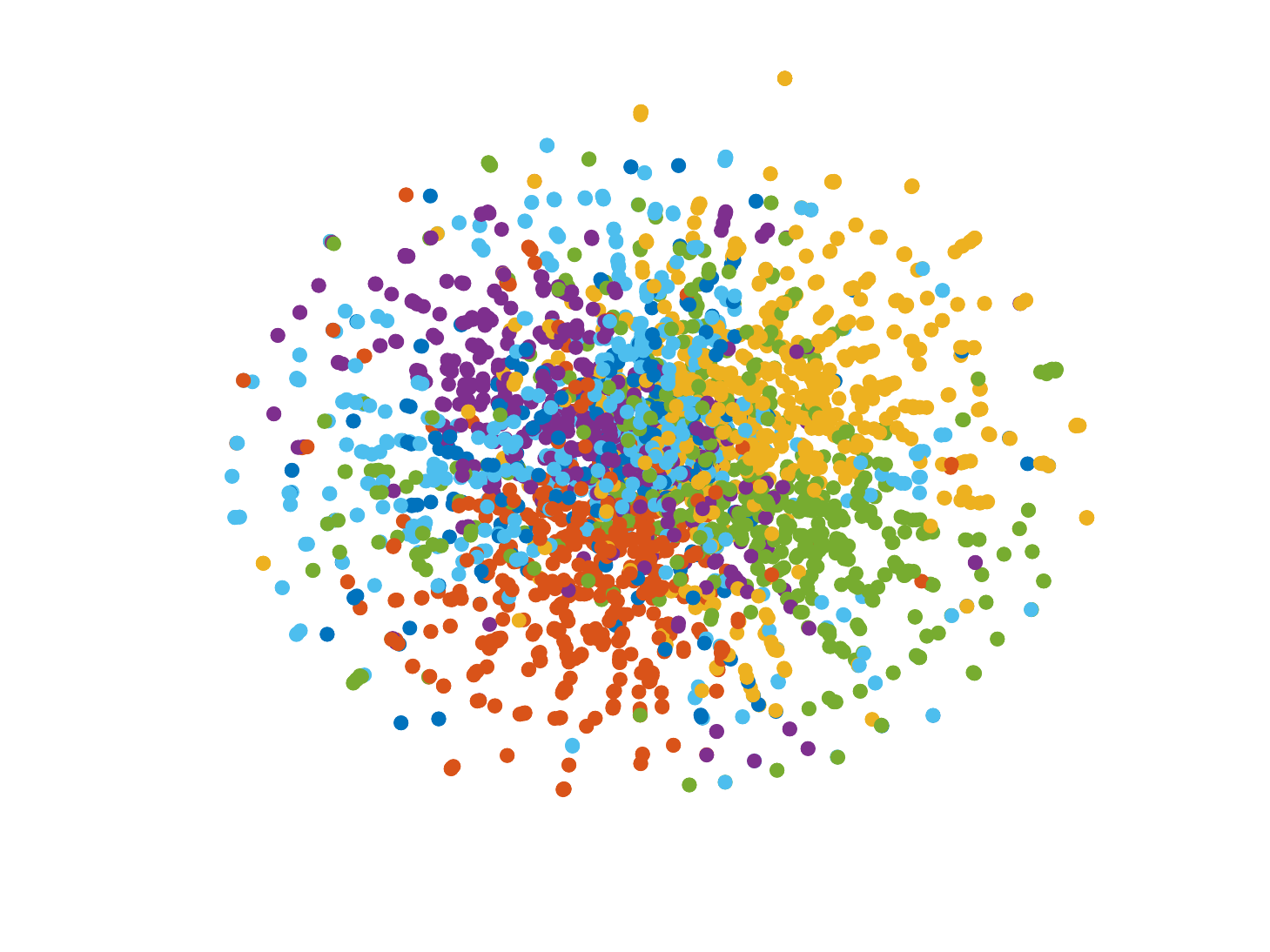}
    }
    \subcaptionbox{$\bm H_1$: From $\mathcal{M}_1$ to $\mathcal{M}_2$}{
        \includegraphics[width=0.23\linewidth]{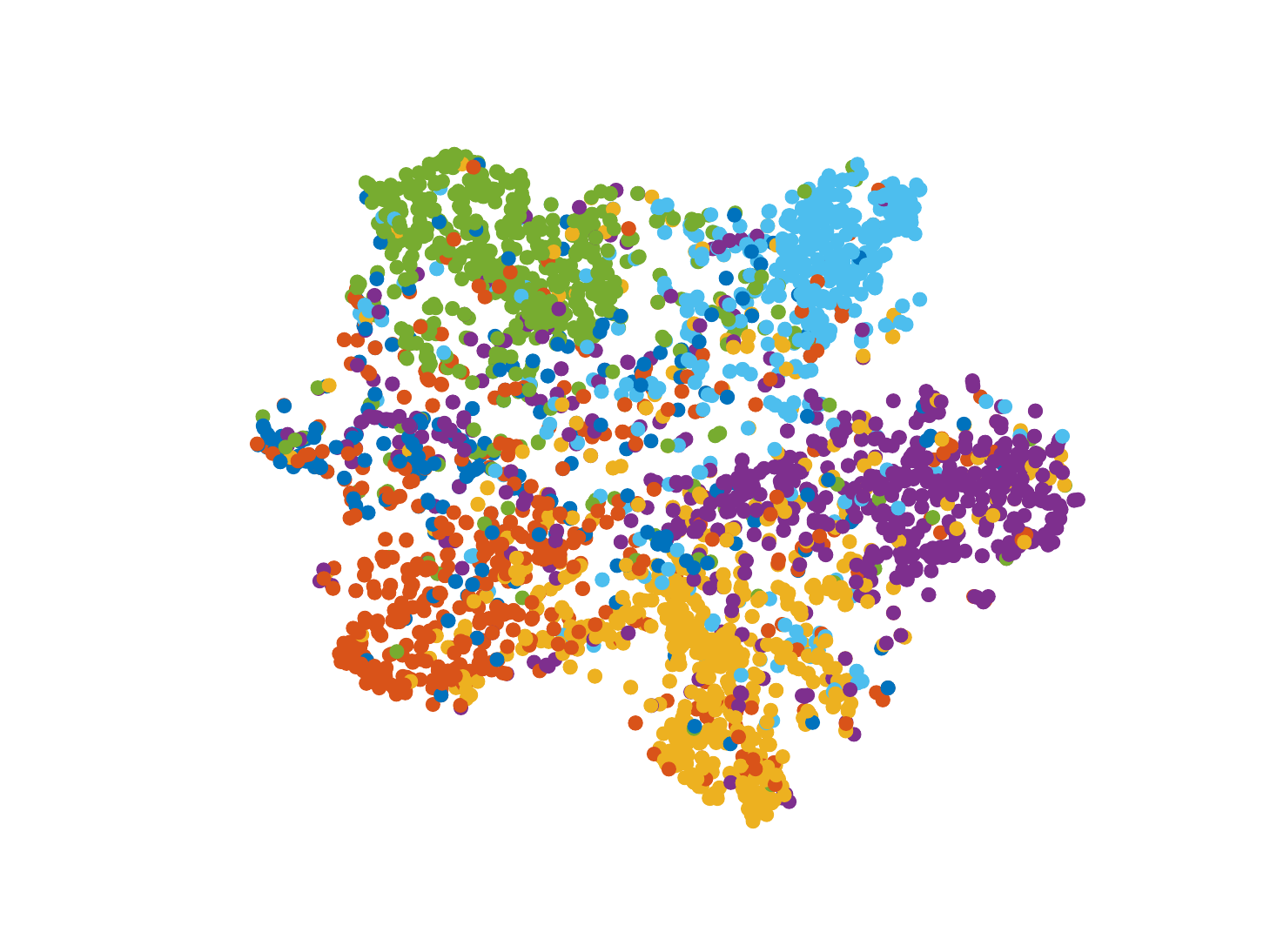}
    }
    \subcaptionbox{$\bm H_2$: From $\mathcal{M}_2$ to $\mathcal{M}_3$}{
        \includegraphics[width=0.23\linewidth]{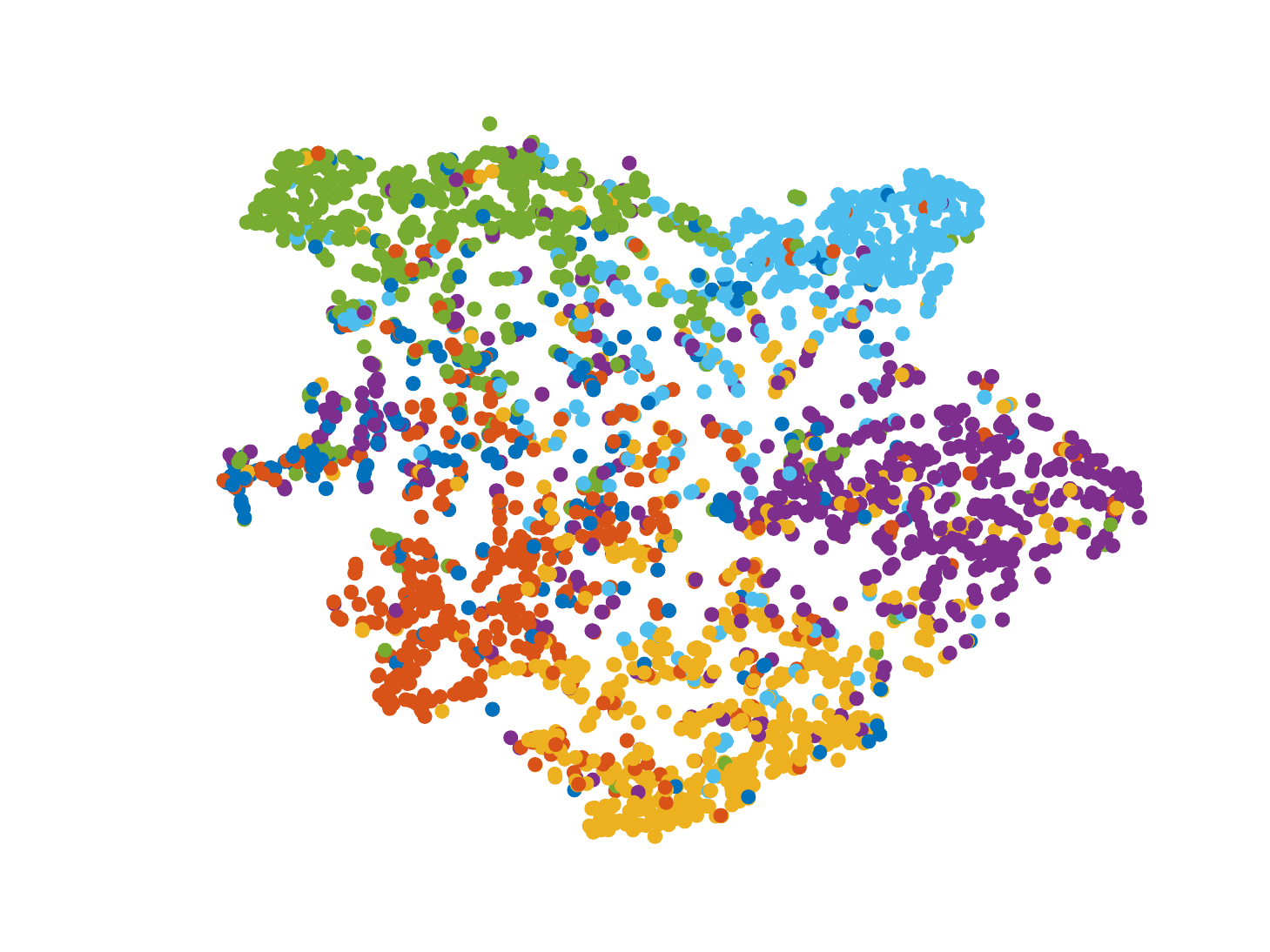}
    }
    \subcaptionbox{$\bm{H}_3$: Output of $\mathcal{M}_3$}{
        \includegraphics[width=0.23\linewidth]{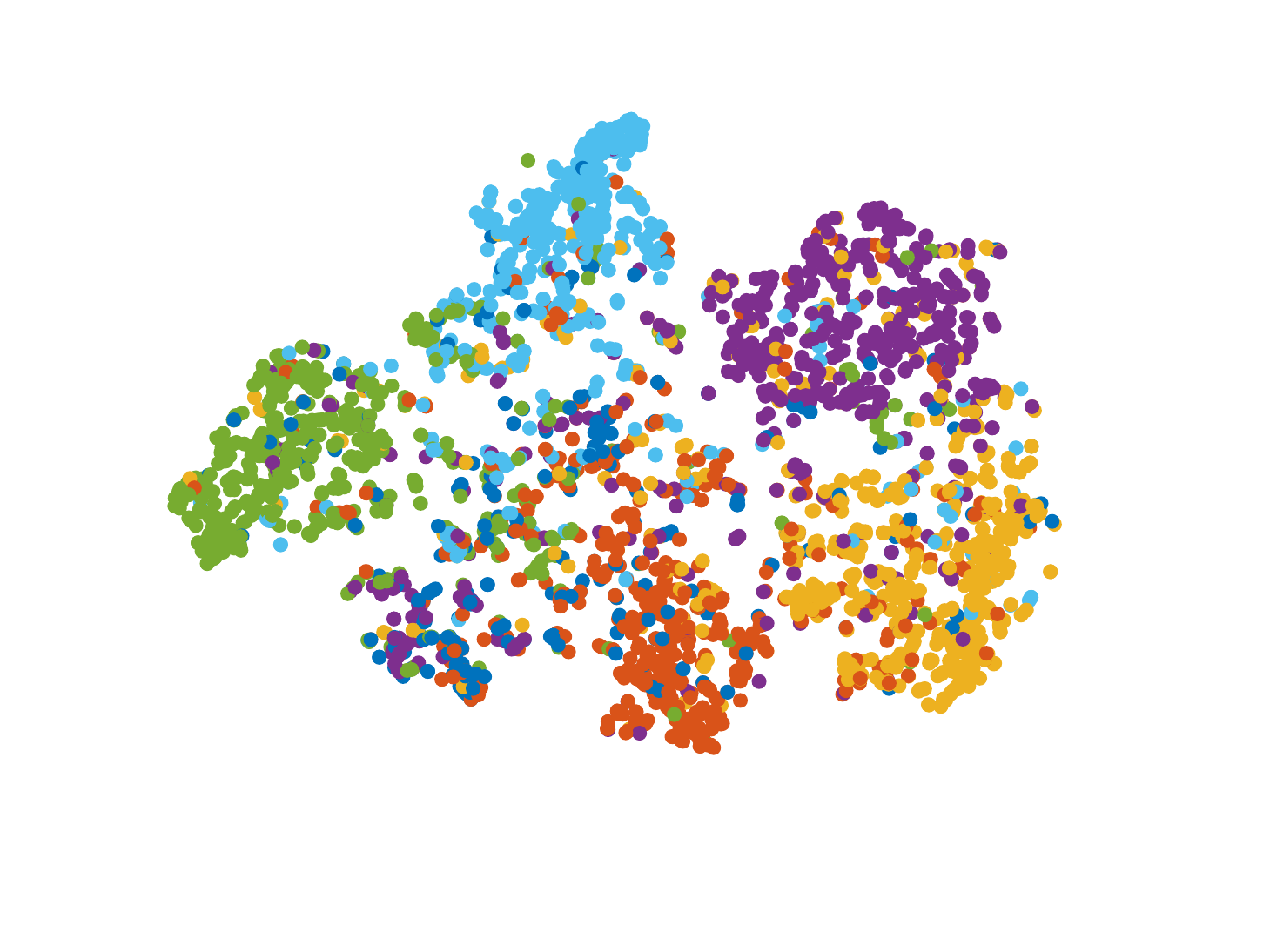}
    }
    \caption{Visualization of a trained SGNN comprised of 3 modules on node classification of Cora and Citeseer.
    The first line shows the visualization of Cora and the bottom line shows the 
    visualization of Citeseer.}
    \label{figure_more_visualization_classification}
\end{figure*}

\subsubsection{Performance}
From Table \ref{table_clustering}, we find that SGNN outperforms in most 
datasets. 
If the released codes could not run on Reddit due to out-of-memory (OOM), 
we put the notation ``N/A'' instead of results. 
In particular, SGNN-BT obtains good improvements on Reddit with high efficiency. 
Specifically speaking, it is about 8\% higher than the well-known GraphSAGE. 
SGNN-FT performs above the average on some datasets. It usually outperforms GraphSAGE but 
fails to exceed SGC. 
Due to the deeper structure caused by multiple modules, 
the performance of SGNN excels the simple GAE. 
It also outperforms SGC due to more non-linearity brought by multiple modules. 
Note that S$^2$GC and SGC are strong competitors, 
while \textbf{SGNN can easily employ them as a base module since they are separable} 
which is shown by SGNN-S$^2$GC. 
It is easy to find that SGNN-S$^2$GC usually achieves similar results compared with S$^2$GC. 
As it is slower than SGNN-BT and the performance improvement is not stable, 
we recommend to use simple SGNN-FT instead of SGNN-S$^2$GC in practice. 
From the table, we find that it is unnecessary to modify S$^2$GC to GAE-S$^2$GC 
since the GAE architecture does not improve the performance of S$^2$GC at all. 
From the ablation experiments, SGNN-BT works better than SGNN-FT, 
which indicates the necessity of the backward training. 

We also investigate how the number of modules $L$ affects the node clustering 
accuracy and the results averaged over 5 runs are reported in 
Table \ref{table_depth}. 
To ensure fairness, we also show the performance of GAE with the same depth 
though deeper GCN and GAE usually return unsatisfied results. 
Note that the neurons of each layers are set as $[256, 128, 64, 32, 16, 16, 16]$, respectively. 
It is not hard to find that SGNN is robust to the depth compared with the conventional GNNs.

\begin{figure*}[t]
    \small
    \centering
    \setlength{\fboxrule}{0.1pt}
    \setlength{\fboxsep}{1pt}
    \subcaptionbox{$\bm H_0$: Input of $\mathcal{M}_1$}{
        \includegraphics[width=0.23\linewidth]{figures/cora-X0.png}
    }
    \subcaptionbox{$\bm H_1$: From $\mathcal{M}_1$ to $\mathcal{M}_2$}{
        \includegraphics[width=0.23\linewidth]{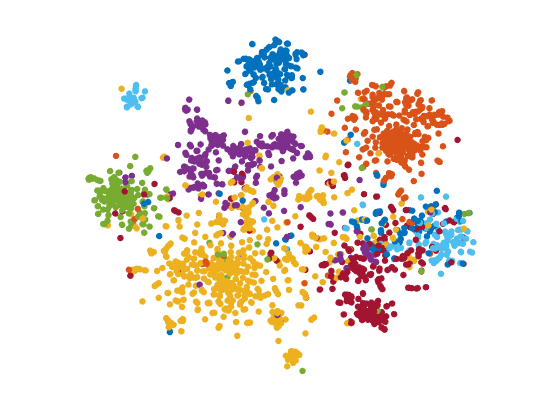}
    }
    \subcaptionbox{$\bm H_2$: From $\mathcal{M}_2$ to $\mathcal{M}_3$}{
        \includegraphics[width=0.23\linewidth]{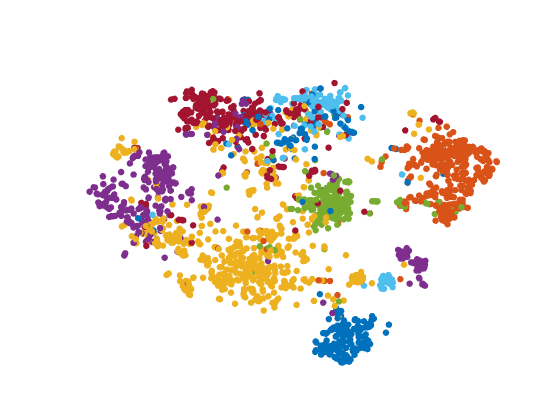}
    }
    \subcaptionbox{$\bm{H}_3$: Output of $\mathcal{M}_3$}{
        \includegraphics[width=0.23\linewidth]{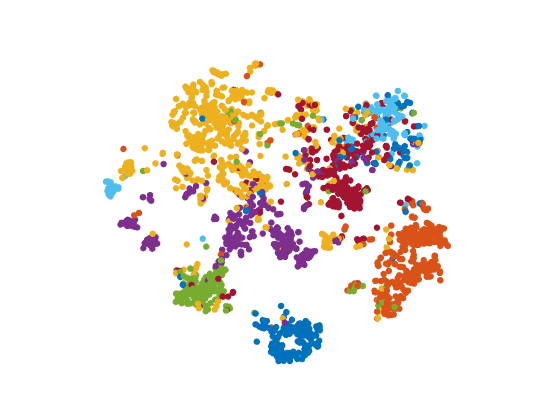}
    }

    \subcaptionbox{$\bm H_0$: Input of $\mathcal{M}_1$}{
        \includegraphics[width=0.23\linewidth]{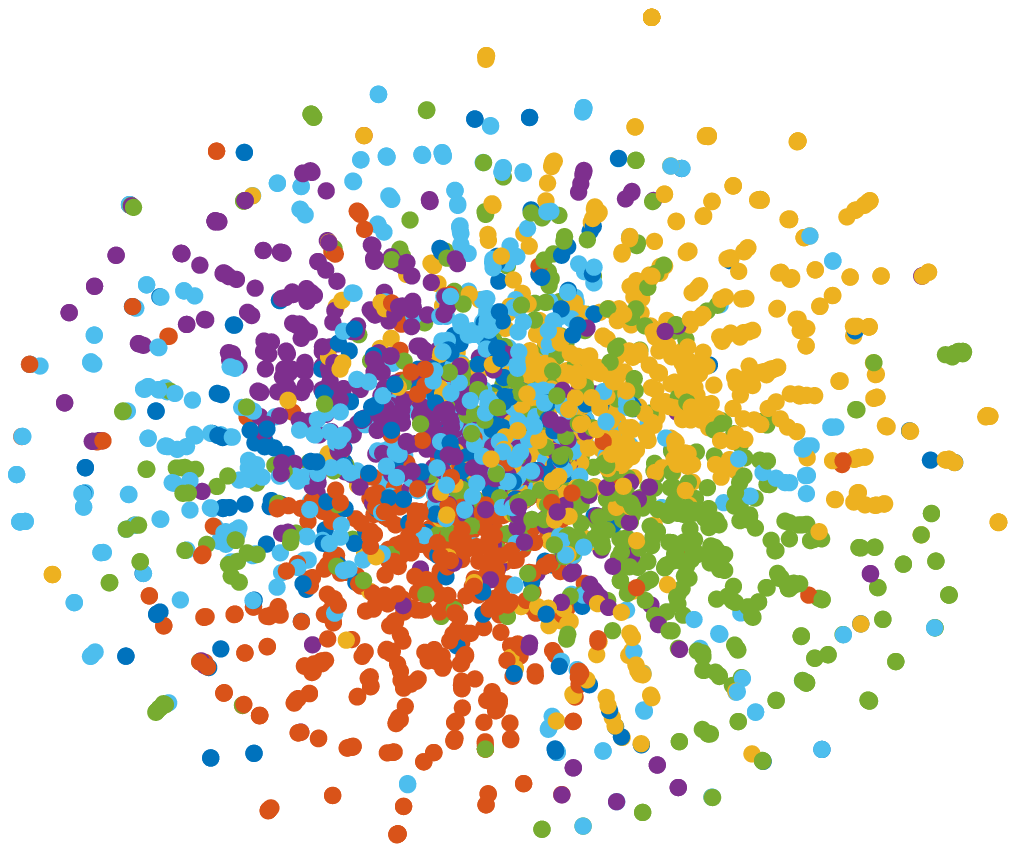}
    }
    \subcaptionbox{$\bm H_1$: From $\mathcal{M}_1$ to $\mathcal{M}_2$}{
        \includegraphics[width=0.23\linewidth]{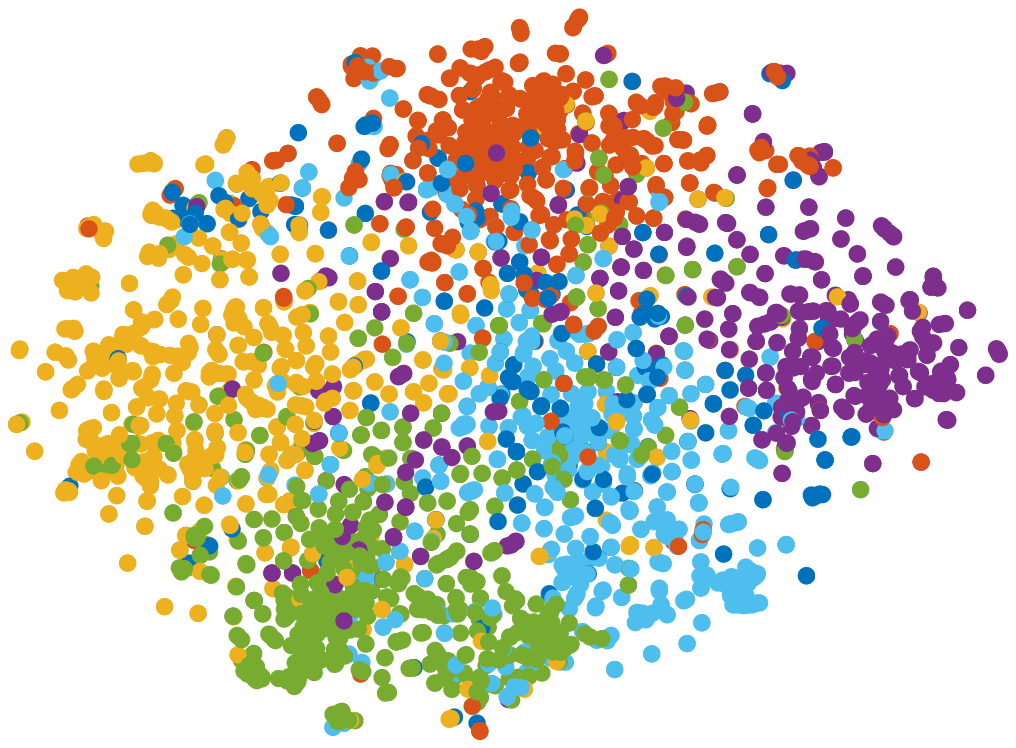}
    }
    \subcaptionbox{$\bm H_2$: From $\mathcal{M}_2$ to $\mathcal{M}_3$}{
        \includegraphics[width=0.23\linewidth]{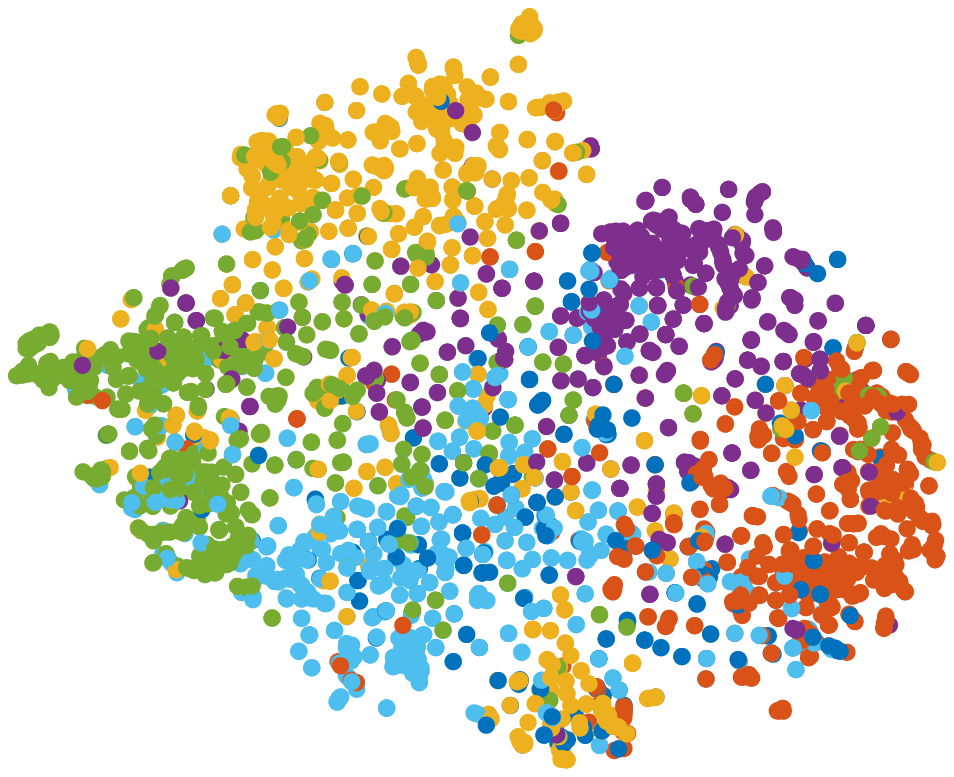}
    }
    \subcaptionbox{$\bm{H}_3$: Output of $\mathcal{M}_3$}{
        \includegraphics[width=0.23\linewidth]{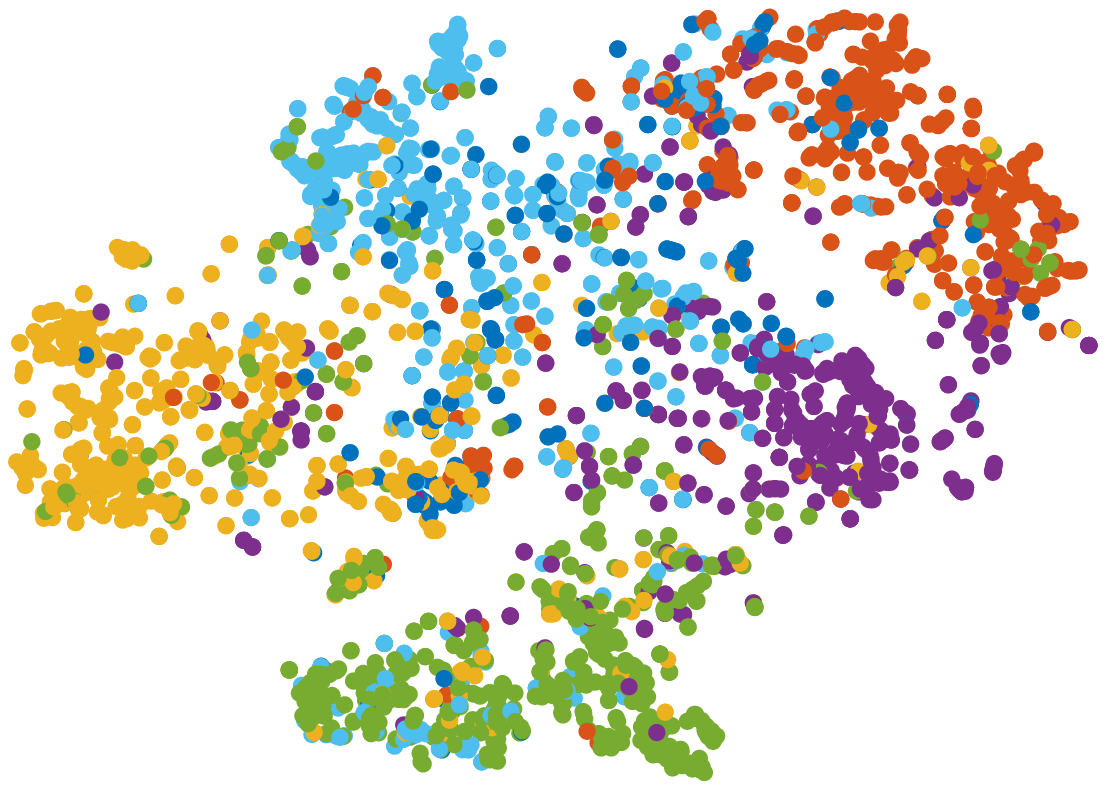}
    }
    \caption{Visualization of a trained SGNN comprised of 3 modules on node clustering of Cora and Citeseer.
    The first line is visualization on Cora and the second line is visualization on Citeseer. }
    \label{figure_more_visualization_clustering}
\end{figure*}

\begin{figure*}[t]
    \small
    \centering
    \setlength{\fboxrule}{0.1pt}
    \setlength{\fboxsep}{1pt}
    \subcaptionbox{$\bm H_1$ after 1 epoch}{
        \includegraphics[width=0.23\linewidth]{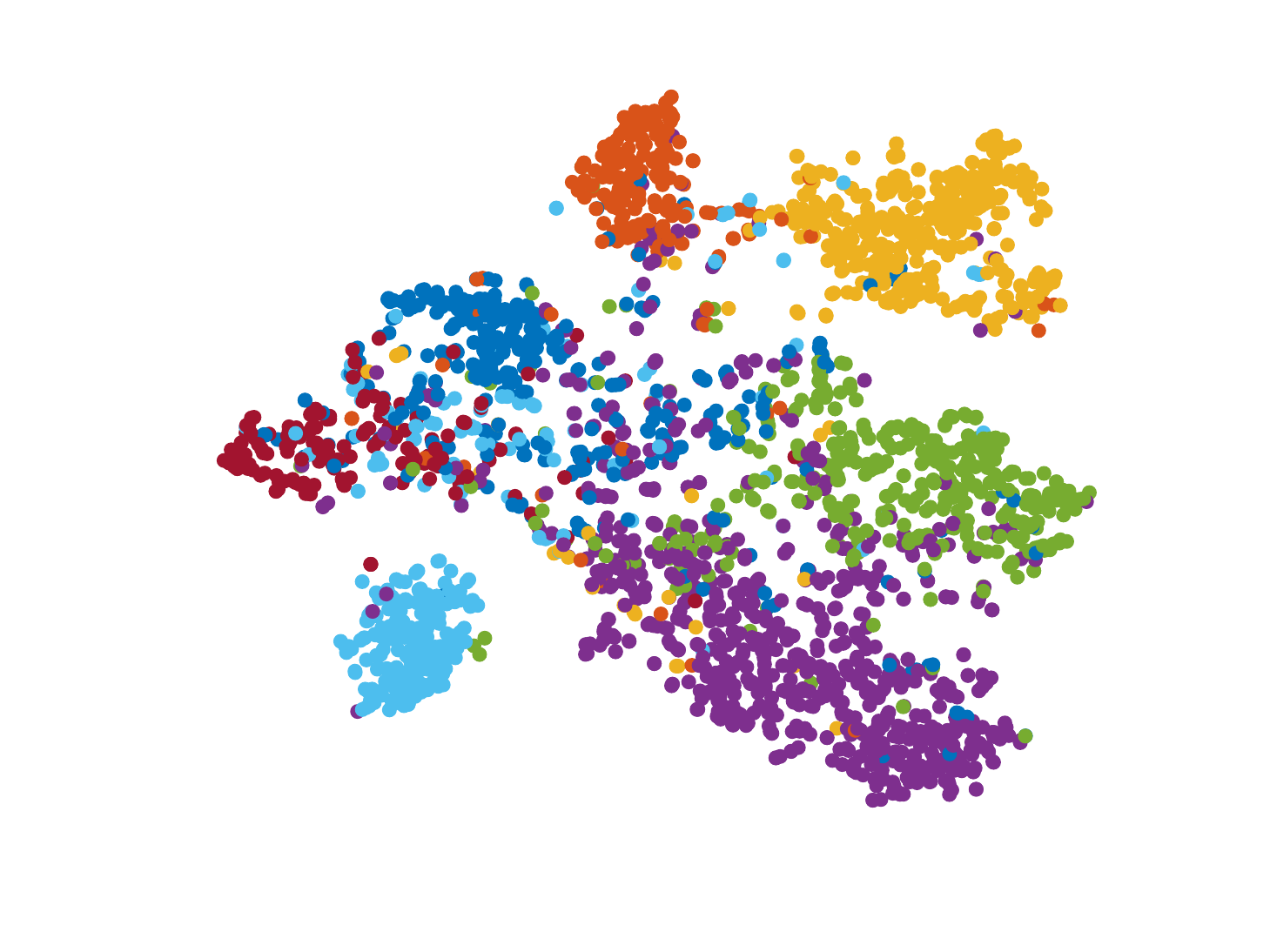}
    }
    \subcaptionbox{$\bm H_1$ after 2 epochs}{
        \includegraphics[width=0.23\linewidth]{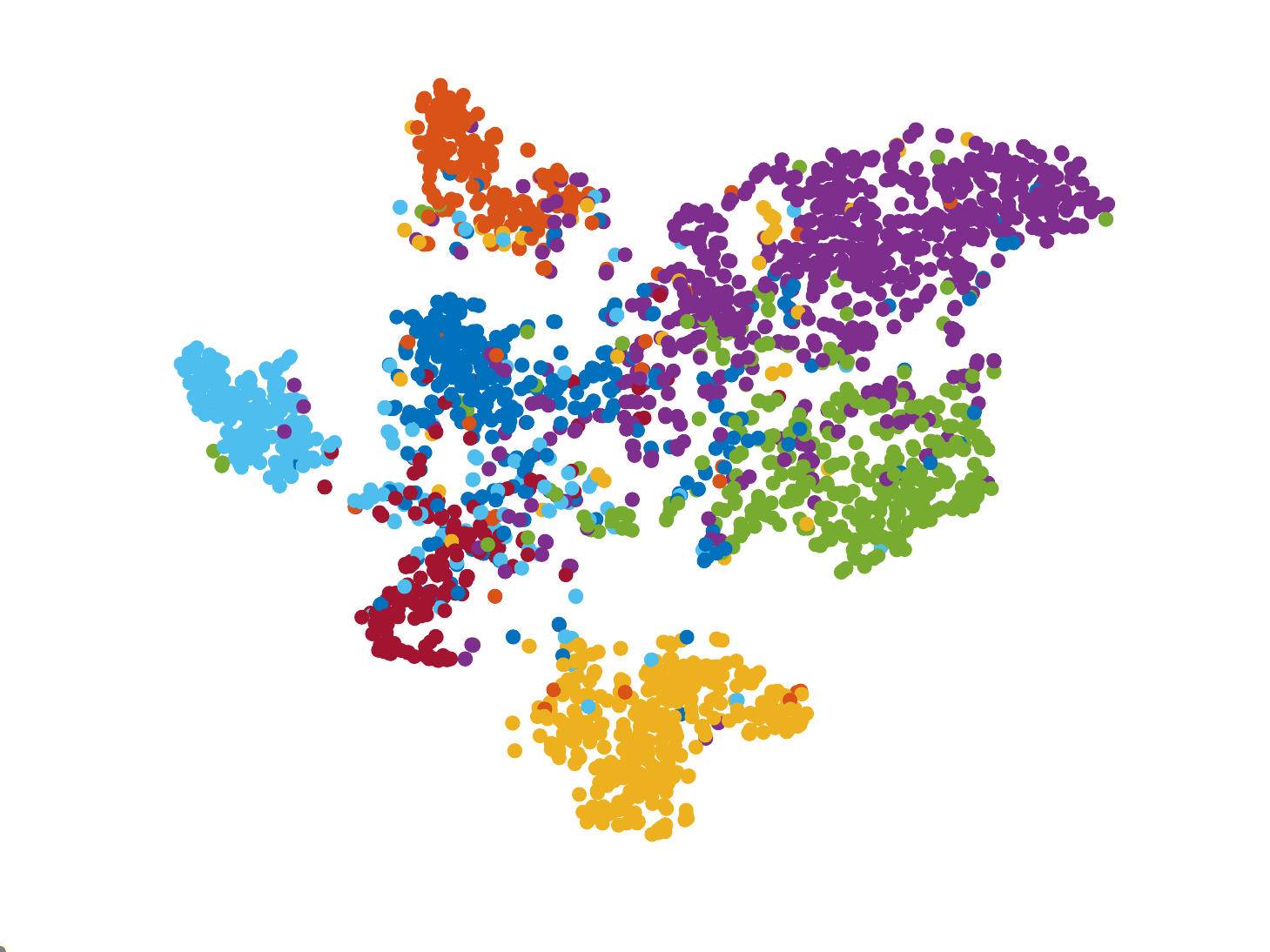}
    }
    \subcaptionbox{$\bm H_1$ after 3 epochs}{
        \includegraphics[width=0.23\linewidth]{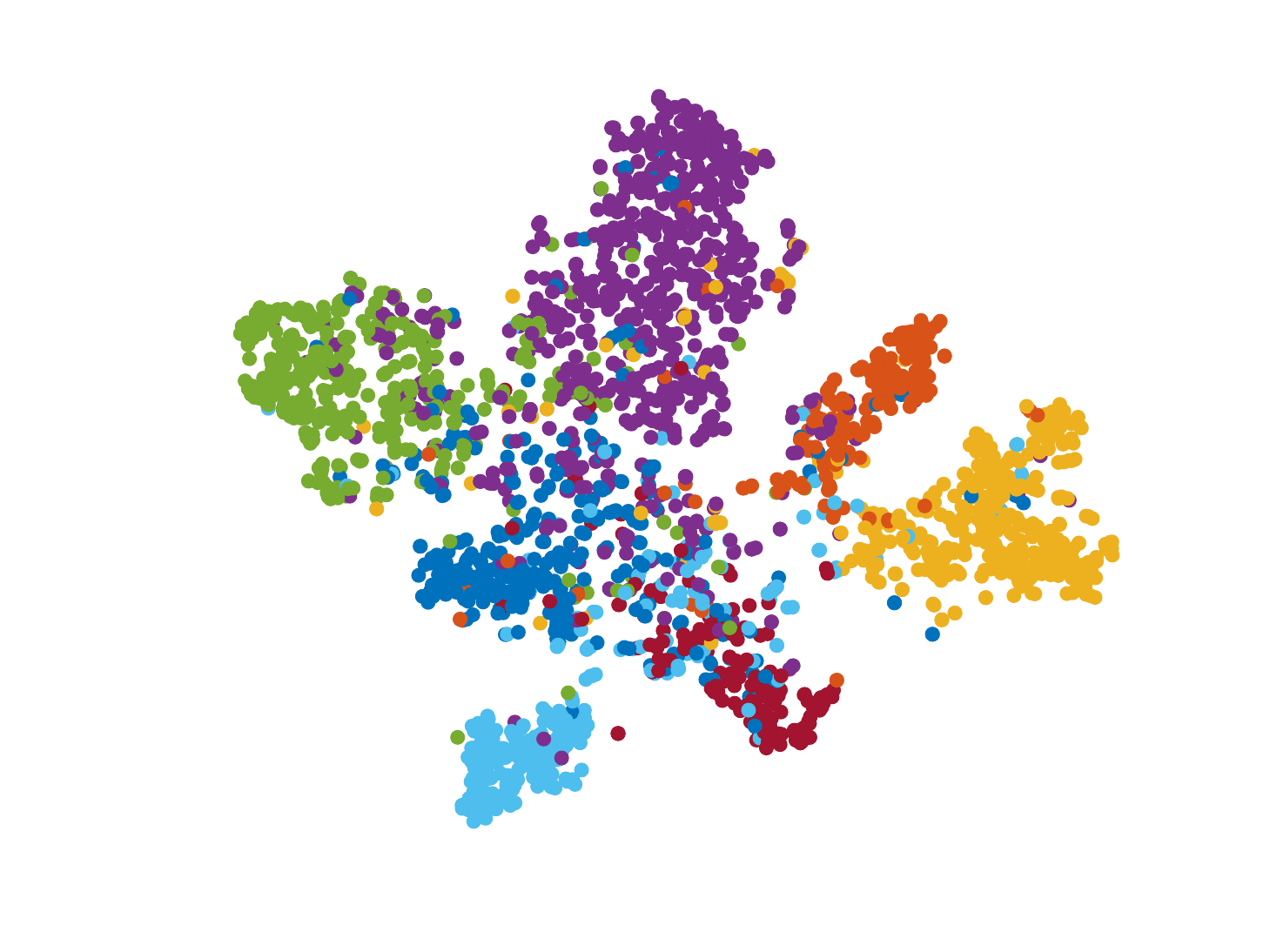}
    }
    \subcaptionbox{$\bm H_1$ after 5 epochs \label{figure_epoch5}}{
        \includegraphics[width=0.23\linewidth]{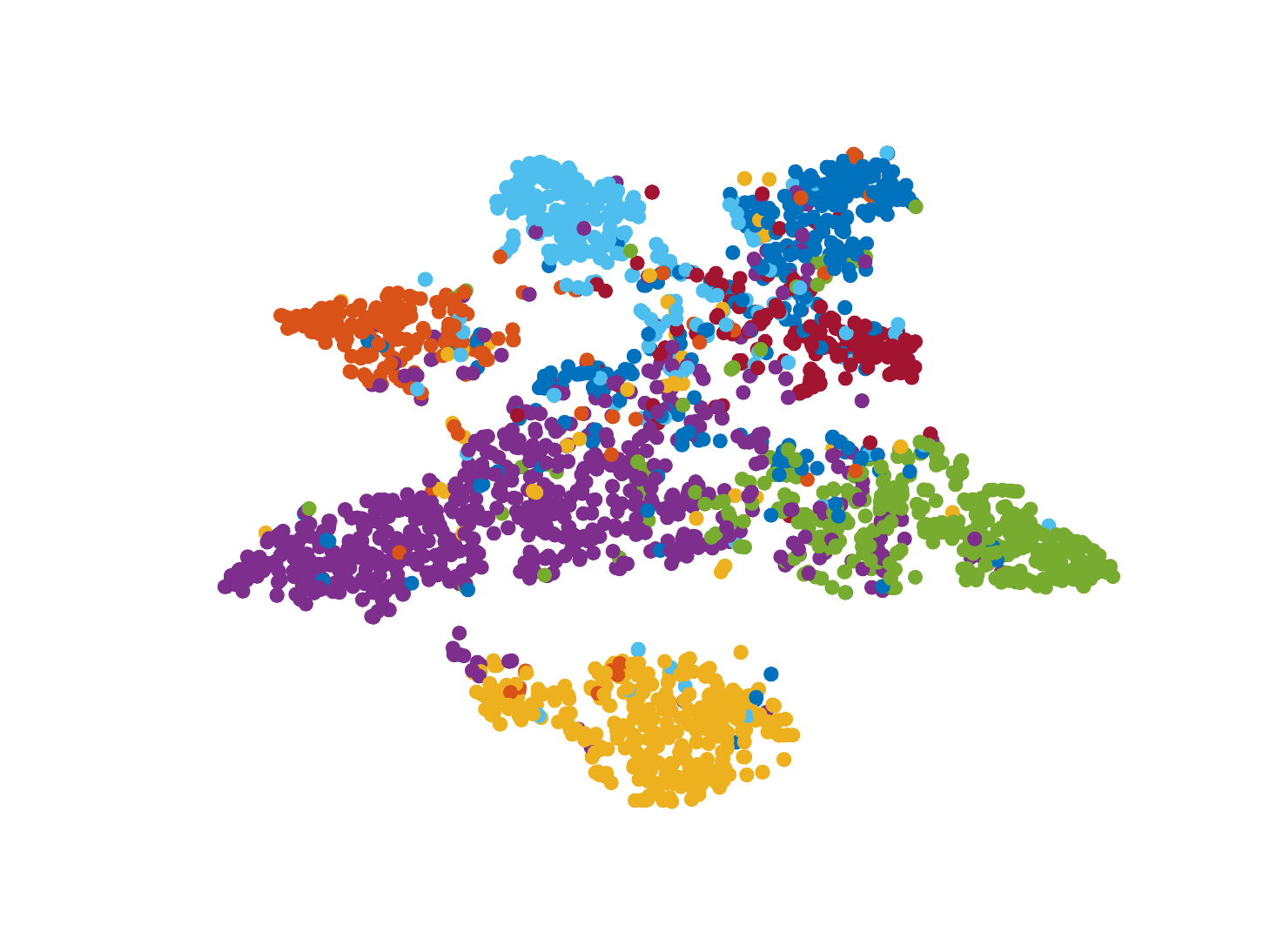}
    }
    \caption{t-SNE Visualization of the output of $\mathcal{M}_1$ from a trained SGNN comprised of 3 modules on node classification of Citeseer.}
    \label{figure_more_visualization_epoch}
\end{figure*}

\begin{figure}[t]
        \small
        \centering
        \setlength{\fboxrule}{0.1pt}
        \setlength{\fboxsep}{1pt}
        \subcaptionbox{\small Cora}{
            \includegraphics[width=0.47\linewidth]{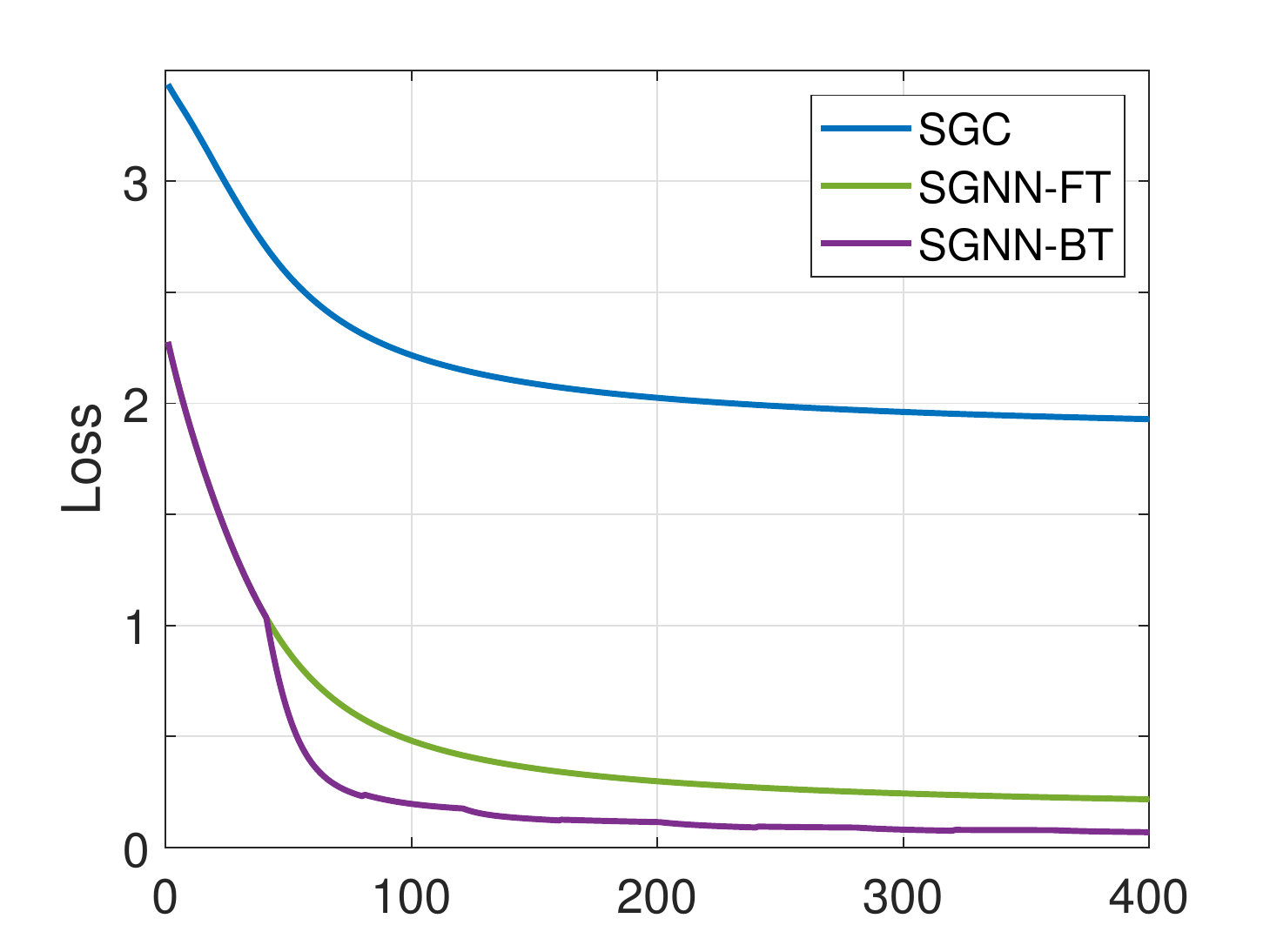}
        }
        \subcaptionbox{\small Pubmed}{
            \includegraphics[width=0.47\linewidth]{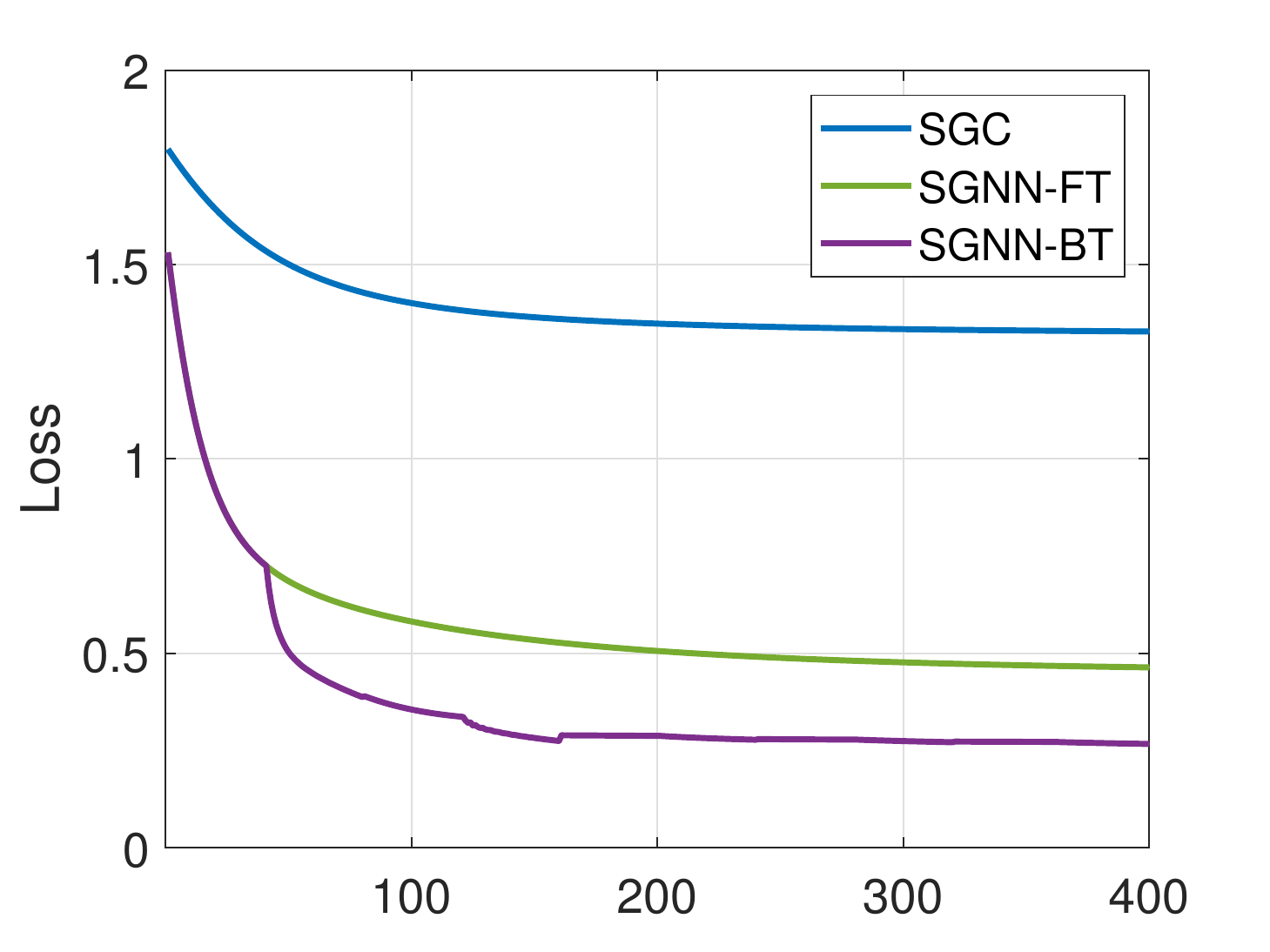}
        }
        \caption{Convergence curve of the \textit{final loss}. 
        The order of SGC is set as $L$. 
        The backward training significantly decreases the final loss and the non-linearity also plays an 
        important role. }
        \label{figure_convergence}
\end{figure}

\subsubsection{Efficiency}
Figure \ref{figure_time} shows the consuming time of several GNNs with higher 
efficiency on Pubmed and Reddit. Instead of neglecting the preprocessing operation, 
we measure the efficiency through a more rational way. 
We record the totally consuming time after loading data into RAM 
and then divide the total number of updating parameters of GNNs. 
The measurement could reflect the real difference of diverse training 
techniques aiming to apply batch-based algorithms to GNN. 
It should be emphasized the reason why SGC is worse than SGNN regarding 
the consuming time. The key point is the different costs of their 
\textit{preprocessing} operation. For an $L$-order SGC, the computation 
cost of $\bm P^L \bm X$ is at least $\mathcal{O}(\|\bm A\|_0 L d)$ while 
SGNN with $L$ first-order modules totally requires $\mathcal{O}(\|\bm A\|_0 \sum_{i=1}^L d_i)$ 
for the same preprocessing operation. 
The metric also provides a fair comparison between SGC and other models 
since the stopping criteria are always different.

\subsection{Node Classification}

\subsubsection{Experimental Setting}
We also conduct experiments of semi-supervised classification on four datasets. 
The split of datasets follows \cite{SGC} which is shown in Table \ref{table_datasets}. 
We compare SGNN against 
GCN \cite{GCN}, GAT \cite{GAT}, DGI \cite{DGI}, APPNP \cite{APPNP}, L2-GCN \cite{L2-GCN} FastGCN \cite{FastGCN}, 
GraphSAGE \cite{GraphSAGE}, Cluster-GCN \cite{Cluster-GCN}, SGC \cite{SGC}, GCNII \cite{GCNII}, 
and S$^2$GC \cite{S2GC}. 
Similarly, we testify SGNN with two different base models, namely SGNN-BT and SGNN-S$^2$GC. 
The experimental settings are same as the experiments of node clustering. 
For GraphSAGE, we use the mean operator by default and some notations are 
added if the extra operators are used. 
On citation networks, 
the learning rate is set as $0.01$, 
while it is $10^{-4}$ on Reddit. 
Since the nodes for training are less than 200 on citation networks, 
we use all training points in each iteration for all methods 
while we sample 256 points as a mini-batch for approaching expected 
features during backward training of SGNN. 
On Reddit, the batch size of all batch-based models is set as 512. 
We do not apply the early stopping criterion used in \cite{GCN} and 
the max iteration follows the setting of SGC.
The embedding dimensions of each module 
are the same as the setting in node clustering.
For the sake of fairness, we report the results obtained by SGNN with two modules 
using first-order operation. 
The forward training loss is defined in Eq. (\ref{loss_GCN}). 
Moreover, 
all compared models share an identical implementation of their mini-batch iterators, 
loss function and neighborhood sampler (when applicable). 
The balance coefficient of $\mathcal{L}_{FT}$ and $\mathcal{L}_{BT}$ is set 
as $1$ by default. 
We report the results averaged over 10 runs on citation datasets and 5 runs on Reddit 
in Table \ref{table_classification_small} and Table \ref{table_classification_reddit}. 
The hyper-parameters are shared for different datasets which are optimized on Cora.

\subsubsection{Performance} 
The results of compared methods in Table \ref{table_classification_small}
are taken from the corresponding papers. When the experimental results are 
missed, we run the publicly released codes and the corresponding records 
are superscripted by $\dag$. 
From Tables \ref{table_classification_small} and \ref{table_classification_reddit}, 
we conclude that SGNN outperforms the models with neighbor sampling 
such as GraphSAGE, FastGCN, and Cluster-GCN on citation networks 
and the performance of SGNN exceeds most models on Reddit. 
On simple citation networks, 
SGNN \textbf{loses the least accuracy compared with other batch-based models}, 
which is close to GCN. 
Owing to the separability of each module, the batch sampling
requires no neighbor sampling and causes no loss of graph information. 
Note that we simply employ the single-layer GCN as the base modules 
in our experiments, while some high-order methods that obtain competitive 
results are also available for SGNN, \textit{e.g.}, SGNN-S$^2$GC. 
As shown by the results of SGNN-S$^2$GC, SGNN can be indeed improved by employing
more complicated separable GNNs as base models.

Although \textbf{some methods achieve preferable results, they either fail to run 
or obtain unsatisfactory results on large-scale datasets}. 
Besides, we also show the comparison of efficiency on node classification task 
in Figure \ref{figure_time}.

\begin{table}[t]
    \centering
    \normalsize
    \renewcommand\arraystretch{1.1}
    \caption{Node Classification Results on Large Datasets}
    \label{table_OGB}
    \begin{tabular}{l c c c c c c c c}
        \hline
        
        \hline
        \multirow{2}{*}{Datasets} & \multicolumn{2}{|c|}{Products} & \multicolumn{2}{c}{Arxiv} \\
        & \multicolumn{1}{|c}{Test Acc} & {Val Acc} & \multicolumn{1}{|c}{Test Acc} & \multicolumn{1}{c}{Val Acc} \\
        \hline
        \hline
        MLP & 61.06 & 75.54 & 55.50 & 57.65 \\
        Softmax & 47.70 & N/A & 52.77 & N/A \\
        GraphSAGE & \underline{78.50} & \textbf{92.24} & 71.49 & \underline{72.77} \\
        Cluster-GCN & \textbf{78.97} & \underline{92.12} & N/A & N/A \\
        \hline
        GCN & 75.64 & 92.00 & \textbf{71.74} & \textbf{73.00} \\
        SGC & 68.87 & N/A & 68.78 & N/A \\
        S$^2$GC & 70.22 & N/A & 70.15 & N/A \\
        SGNN-FT & 68.10 & 86.26 & 65.25 & 64.00 \\
        SGNN-BT & 74.44 & 91.13 & \underline{71.57} & 71.66 \\
        \hline

        \hline
        
    \end{tabular}
\end{table}

\subsubsection{Visualization to Show Impact of the Decoupling}
In Figure \ref{figure_visualization}, we visualize the output of a 3-module SGNN and 
a 3-layer GCN to directly show that \textit{the decoupling would not cause the trivial features}, 
which corresponds to the theoretical conclusion in Section \ref{section_theo}. 
To show the benefit of the non-linearity brought by SGNN 
and the backward training, 
the convergence curves of SGC, SGNN-FT, and SGNN-BT are shown in Figure \ref{figure_convergence}. 
Note that the figure shows the variation of the \textit{final loss}. 
In SGNN, the final loss is the loss of $\mathcal{M}_L$, while it is the unique training loss in SGC. 
SGC with $L$-order graph operation is used. 
From this figure, we can conclude that: 
(1) The non-linearity does lead to a better loss value; 
(2) The backward training significantly decreases the loss. 
In summary, \textit{the decoupling empirically does not 
cause the negative impact. }

\subsection{Visualization} \label{section_visualization}
We also provide more visualizations 
in Figures \ref{figure_more_visualization_classification} and \ref{figure_more_visualization_clustering}. 
We run SGNN with 3 GNN modules and visualize the input and output of 
$\mathcal{M}_1$, $\mathcal{M}_2$, and $\mathcal{M}_3$ through $t$-SNE 
on Cora and Citeseer, for node clustering and node classification. 
The purpose of these two figures is to empirically investigate \textit{whether 
the decoupling would cause the accumulation of residuals and errors}. 
The experimental results support the theoretical results that are provided in Section \ref{section_theo}. 
One may concern the impact of $\eta$ (trade-off coefficient between $\mathcal{L}_{FT}$ and $\mathcal{L}_BT$) 
on the performance. 
We testify SGNN with different $\eta$ from $\{10^{-5}, 10^{-3}, 10^{-1}, 10^{1}, 10^3, 10^{5}\}$ and 
find that \textbf{$\eta=10^3$ usually leads to good results}. 
Accordingly, we only report results SGNN with $\eta=10^3$ in this paper. 
Moreover, we show the impact of $\eta$ to node clustering on Cora and Citeseer 
in Figure \ref{figure_eta}.

Moreover, we show the output of $\mathcal{M}_1$ of different periods in Figure 
\ref{figure_more_visualization_epoch}, in order to show the impact of the backward training. 
From the figure, we find that BT indeed affects the latent features, 
which is particularly apparent in Figure \ref{figure_epoch5}.

\subsection{Experiments on OGB Datasets} \label{section_OGB}

We further show some experiments of node classification on two OGB datasets, OGB-Products and OGB-Arxiv, 
which are downloaded from \url{https://ogb.stanford.edu/docs/nodeprop/}. 
The OGB-Products contains more than 2 million nodes and OGB-Arxiv contains 
more than 150 thousand nodes.  

It should be emphasized that we only use the simple single-layer GCN as the 
base module of SGNN. The performance can be further improved by incorporating 
different models such as GCNII, GIN, \textit{etc}. 
In particular, we only tune hyper-parameters on Arxiv, and we simply report 
results of SGNN with settings from Reddit. 

\textbf{Remark}: One may concern that the proposed SGNN cannot achieve the state-of-the-art results 
like node clustering. It should be emphasized that \textbf{an SGNN with $L$ modules should be regarded as a variant of an
$L$-layer GNN}. So it is more fair to compare SGNN with GCN. 
From Tables \ref{table_classification_small}, \ref{table_classification_reddit}, and \ref{table_OGB}, 
we can find that the performance of SGNN can approach GCN with high efficiency. 
A main reason why SGNN cannot outperforms other models on supervised tasks like node clustering 
is the difficulty of designing proper training losses for middle modules. 
It essentially originates from the black-box property of neural networks, 
\textit{i.e.}, \textbf{what kind of latent features is preferable for deeper layers}. 
From the experimental results, we find that it is not the optimal scheme to simply set 
the final supervised loss as the training loss of middle module. 
Instead, on node clustering, we prove that the greedy strategy would not accumulate the error 
and the experimental results validate the theoretical conclusion. 

\begin{figure}[t]
    \small
    \centering
    \setlength{\fboxrule}{0.1pt}
    \setlength{\fboxsep}{1pt}
    \subcaptionbox{Cora}{
        \includegraphics[width=0.45\linewidth]{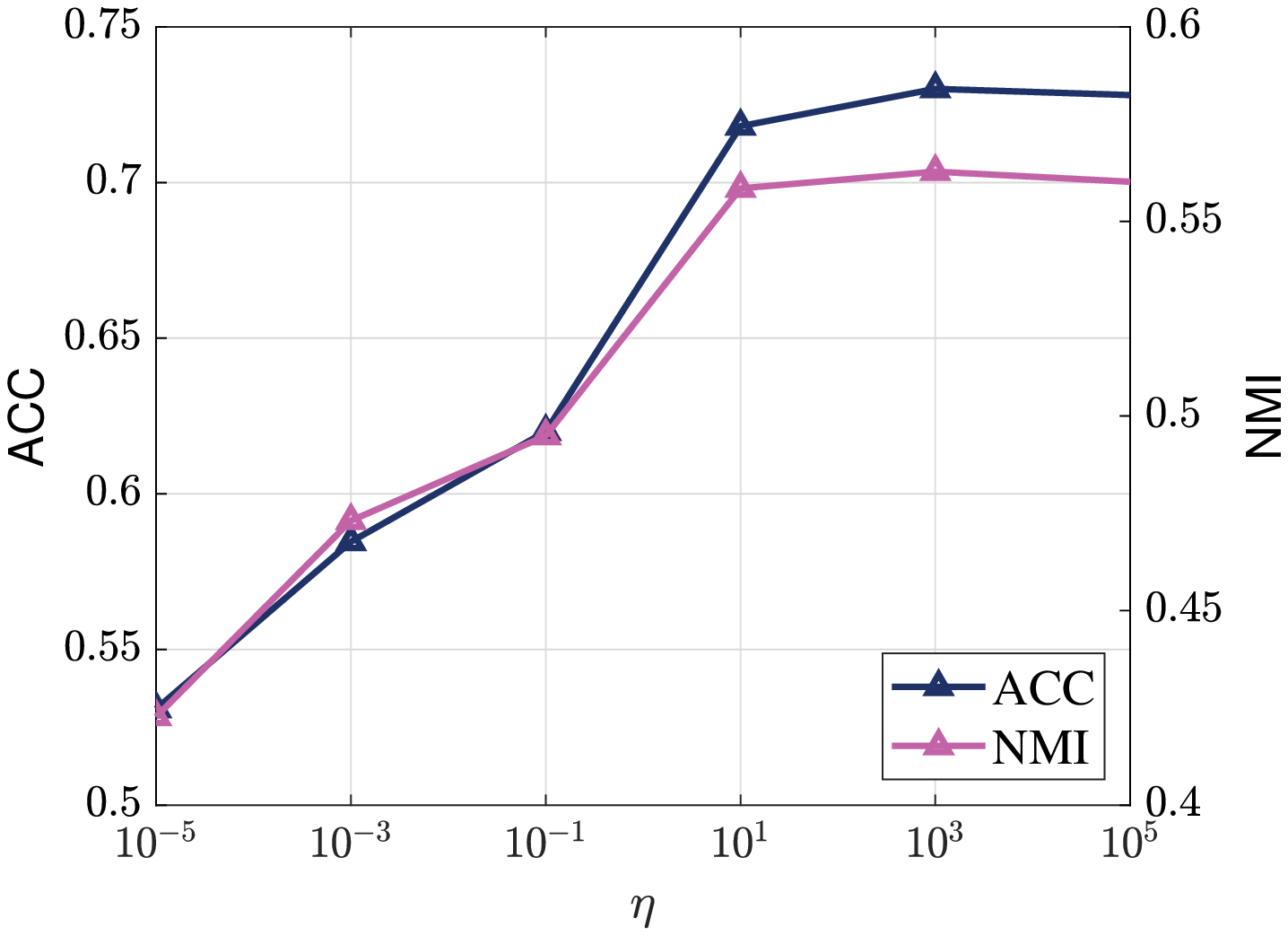}
    }
    \subcaptionbox{Citeseer}{
        \includegraphics[width=0.45\linewidth]{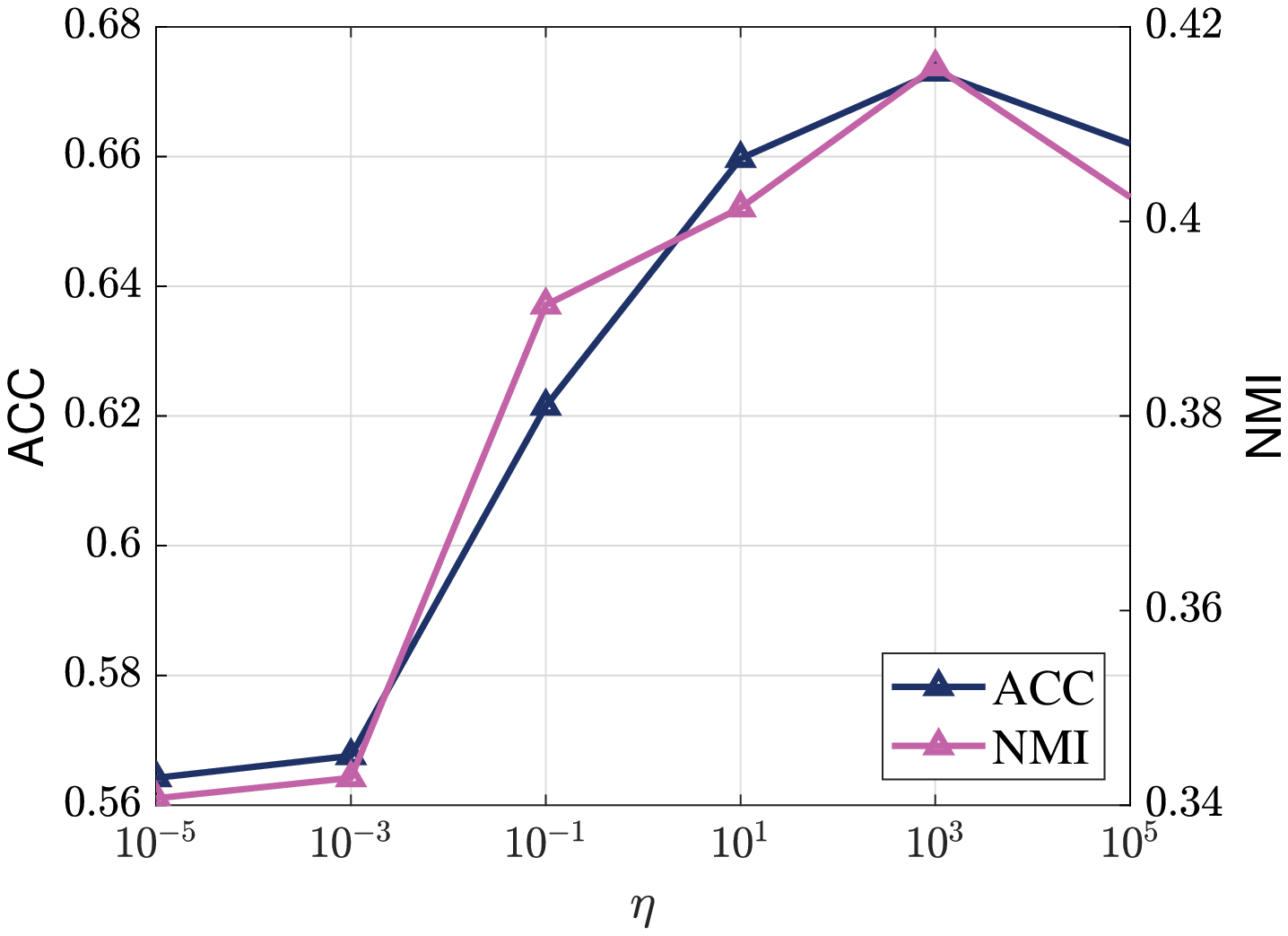}
    }
    \caption{Impact of $\eta$ to node clustering on Cora and Citeseer.}
    \label{figure_eta}
\end{figure}

\section{Conclusion and Future Works}
In this paper, 
we propose the Stacked Graph Neural Networks (SGNN). 
We first decouple a multi-layer GNN into multiple simple GNNs, which is 
formally defined as separable GNNs in our paper to ensure the availability 
of batch-based optimization without loss of graph information. 
The bottleneck of the existing stacked models is that the information delivery 
is only unidirectional, and therefore a backward training mechanism is 
developed to make the former modules perceive the latter ones. 
We also theoretically prove that the residual of linear SGNN would not accumulate 
in most cases for unsupervised graph tasks. 
The theoretical and experimental results show that the proposed 
framework is more than an efficient method and  
it may deserve further investigation in the future. 
The theoretical analysis focuses on linear SGNN and the generalization bound 
is also not investigated in this paper. Therefore, they will be the core of our 
future work. 
Moreover, since $\mathcal{L}_{FT}$ could be any losses and SGNN fails to achieve state-of-the-art results, 
how to set the most appropriate training loss, especially on supervised tasks, for each module will be also a crucial topic in our future works. 
It may help us understand how neural networks work.




\section{Proofs}

\subsection{Lemma for Proofs} \label{appendix_lemma}
\begin{lemma} \label{lemma_commute}
    For two given symmetric matrices $\bm A$ and $\bm B$, 
    $\bm A$ and $\bm B$ share the same eigenspace if and only if 
    $\bm A$ and $\bm B$ commute. 
\end{lemma}
\begin{proof}
    First, if $\bm A$ and $\bm B$ share the same eigenspace, then there exists 
    $\bm P$ such that 
    \begin{equation}
        \bm A = \bm P \bm \Lambda_A \bm P^T, \bm B = \bm P \bm \Lambda_B \bm P^T .
    \end{equation}
    Accordingly, we have 
    $\bm A \bm B = \bm P \bm \Lambda_A \bm \Lambda_B \bm P^T = \bm {B A}$. 

    Then, we turn to prove the converse. If $\bm A$ and $\bm B$ commute, 
    suppose that $\bm A \bm u = \lambda \bm u$ and then 
    \begin{equation}
        \bm{ABu} = \bm{BAu} = \lambda \bm{Bu}. 
    \end{equation}
    Apply eigenvalue decomposition, we have $\bm A = \bm U \bm \Lambda \bm U^T$. 
    Note that if $\bm{A} \bm u_1 = \lambda_1 \bm u_1$, 
    $\bm{A} \bm u_2 = \lambda_2 \bm u_2$, and $\lambda_1 \neq \lambda_2$, 
    then $\bm u_2^T \bm B \bm u_1 = 0$ since $\bm B \bm u_1$ is also an eigenvector 
    associated with $\lambda_1$. 
    Therefore, $\bm U^T \bm{B U}$ is a block-diagonal matrix, \textit{i.e.}, 
    \begin{equation}
        \bm U^T \bm B \bm U = 
        \left [
        \begin{array}{c c c c}
            \bm C_1 & & & \\
            & \bm C_2 & & \\
            & & \ddots & \\
            & & & \bm C_k \\
        \end{array}
        \right ] .
    \end{equation}
    Apply eigendecomposition to $\bm C_i$, 
    \begin{equation}
        \bm C_i = \bm U_i \bm \Lambda_i \bm U_i^T. 
    \end{equation}
    Denote 
    \begin{equation}
        \bm T = 
        \left [
        \begin{array}{c c c c}
            \bm U_1 & & & \\
            & \bm U_2 & & \\
            & & \ddots & \\
            & & & \bm U_k \\
        \end{array} ,
        \right ] 
    \end{equation}
    and $\bm V = \bm U \bm T$, which leads to 
    \begin{equation}
        \bm T^T \bm U^T \bm B \bm U \bm T = \bm \Lambda_B ~~ \textrm{and} ~~ 
        \bm T^T \bm U^T \bm A \bm U \bm T = \bm \Lambda_A
    \end{equation}
    where $\bm V^T \bm{V} = \bm I$. Hence, the lemma is proved. 
\end{proof}

\subsection{Proof of Theorem \ref{theo_GAE}} \label{appendix_proof_1}
\begin{proof}
    Use the notation $\ell(\cdot, \cdot, \cdot)$ as the reconstruction loss 
    \begin{equation}
    \ell(\bm P, \bm X, \bm W) = \|\bm P - \bm{P X W} \bm W^T \bm X^T \bm P^T\|. 
    \end{equation}
    According to the conditions, we define 
    \begin{equation}
        \bm E = \bm P - \bm X \bm X^T \Rightarrow \|\bm E\| = \varepsilon. 
    \end{equation}
    Apply SVD, we can factorize $\bm X$ as 
    \begin{equation}
        \bm X = \bm U_o \bm \Sigma_o \bm V_o^T + \bm U_e \bm \Sigma_e \bm V_e^T, 
    \end{equation}
    where $e = d-o$. Clearly, we have $\bm V_o^T \bm V_e = \bm 0$ and thus 
    \begin{equation}
        \begin{split}
            \bm P = \bm X \bm X^T + \bm E = \bm U_k \bm \Sigma_k^2 \bm U_k^T + \bm U_e \bm \Sigma_e^2 \bm U_e^T + \bm E . 
        \end{split}
    \end{equation}
    Therefore, $\bm H \bm H^T$ can be written as 
    \begin{equation}
        \bm P (\bm U_o \bm \Sigma_o \bm V_o^T + \bm U_e \bm \Sigma_e \bm V_e^T) \bm W \bm W^T (\bm V_o \bm \Sigma_o \bm U_o^T + \bm V_e \bm \Sigma_e \bm U_e^T) \bm P .
    \end{equation}
    Let $\bm W_0$ be a valid solution as 
    \begin{equation}
        \bm W_0 = 
        \begin{cases}
            \bm W_0  & s.t.~ \bm V_k^T \bm W_0 = \bm \Sigma_k^{-2} ~~ {\rm if ~ rank}(\bm X) = k \\
            [\bm W_r; \bm 0]  & s.t. ~ \bm V_r^T \bm W_r = \bm \Sigma_r^{-2} ~~ {\rm if ~ rank}(\bm X) = r \leq k
        \end{cases}
        . 
    \end{equation}
    By the above definition, $\bm V_e^T \bm W = \bm 0$.
    Therefore, with ${\rm rank}(\bm X) > k$, 
    \begin{align}
        \notag & \|\bm H \bm H^T - \bm P\| \\
        \notag = & \| \bm U_k \bm \Sigma_k^3 \bm V_k \bm W_0 \bm W_0^T \bm V_k \bm \Sigma_k^3 \bm U_k^T \\ 
        \notag & ~ + \bm E \bm U_k \bm \Sigma_k \bm V_k \bm W_0 \bm W_0^T \bm V_k \bm \Sigma_k^3 \bm U_k^T \\
        \notag & ~ + \bm U_k \bm \Sigma_k^3 \bm V_k \bm W_0 \bm W_0^T \bm V_k \bm \Sigma_k \bm U_k^T \bm E \\ 
        \notag & ~ + \bm E \bm U_k \bm \Sigma_k \bm V_k \bm W_0 \bm W_0^T \bm V_k \bm \Sigma_k \bm U_k^T \bm E - \bm P\| \\
        \notag = & \|\bm U_k \bm \Sigma_k^2 \bm U_k^T + \bm E \bm U_k \bm U_k^T + \bm U_k \bm U_k^T \bm E + \bm E \bm U_k \bm \Sigma_k^{-2} \bm U_k^T \bm E \\ 
        \notag & ~ - \bm U_k \bm \Sigma_k^2 \bm U_k^T - \bm U_e \bm \Sigma_e^2 \bm U_e^T - \bm E \| \\
        \notag = & \| \bm E \bm U_k \bm U_k^T + \bm U_k \bm U_k^T \bm E + \bm E \bm U_k \bm \Sigma_k^{-2} \bm U_k^T \bm E \\ 
        \notag & ~ - \bm U_e \bm \Sigma_e^2 \bm U_e^T - \bm E\| \\
        \notag \leq & \| \bm E \bm U_k \bm U_k^T + \bm U_k \bm U_k^T \bm E - \bm E\| + \|\bm E \bm U_k \bm \Sigma_k^{-2} \bm U_k^T \bm E\| \\ 
        \label{eq_raw_bound_full_rank} & ~ + \|\bm U_e \bm \Sigma_e^2 \bm U_e^T\| . 
    \end{align}
    Similarly, if ${\rm rank}(\bm X) = r \leq k$, 
    \begin{align}
        \notag & \|\bm H \bm H^T - \bm P\| \\
        = & \| \bm E \bm U_r \bm U_r^T + \bm U_r \bm U_r^T \bm E + \bm E \bm U_r \bm \Sigma_r^{-2} \bm U_r^T \bm E - \bm E\| \\
        \leq & \| \bm E \bm U_r \bm U_r^T + \bm U_r \bm U_r^T \bm E - \bm E\| + \|\bm E \bm U_r \bm \Sigma_r^{-2} \bm U_r^T \bm E\| . 
    \end{align}
    Now we focus on the general case, ${\rm rank}(\bm X) > k$ and the conclusion 
    can be easily extended into the low-rank case. Note that 
    \begin{equation}
        \| \bm E (\bm U_k \bm U_k^T - \frac{1}{2} \bm I)\| = \|(\bm U_k \bm U_k^T - \frac{1}{2} \bm I) \bm E\| , 
    \end{equation}
    and the first term can be written as 
    \begin{align*}
        & \| \bm E \bm U_k \bm U_k^T + \bm U_k \bm U_k^T \bm E - \bm E\|^2 \\
        = & \| \bm E (\bm U_k \bm U_k^T - \frac{1}{2} \bm I ) + (\bm U_k \bm U_k^T - \frac{1}{2} \bm I) \bm E \|^2 \\
        = & \|\bm E (\bm U_k \bm U_k^T - \frac{1}{2} \bm I)\|^2 + \|(\bm U_k \bm U_k^T - \frac{1}{2} \bm I) \bm E\|^2 \\ 
        & + 2 \langle \bm E (\bm U_k \bm U_k^T - \frac{1}{2} \bm I), (\bm U_k \bm U_k^T - \frac{1}{2} \bm I) \bm E \rangle \\
        = & 2 \|\bm E (\bm U_k \bm U_k^T - \frac{1}{2} \bm I)\|^2 + 2 s \|\bm E (\bm U_k \bm U_k^T - \frac{1}{2} \bm I)\|^2 \\
        = & 2(1+s) \|\bm E (\bm U_k \bm U_k^T - \frac{1}{2} \bm I)\|^2 ,
    \end{align*}
    where 
    \begin{equation}
        s = \frac{\langle \bm E (\bm U_k \bm U_k^T - \frac{1}{2} \bm I), (\bm U_k \bm U_k^T - \frac{1}{2} \bm I) \bm E \rangle}{\|\bm E (\bm U_k \bm U_k^T - \frac{1}{2} \bm I)\|^2} .
    \end{equation}
    Due to that 
    \begin{align*}
        & \| \bm E (\bm U_k \bm U_k^T - \frac{1}{2} \bm I)\|^2 = {\rm tr}(\bm E^2 (\bm U_k \bm U_k^T - \frac{1}{2} \bm I)) \\
        & = {\rm tr}(\bm E^2 (\bm U_k \bm U_k^T \bm U_k \bm U_k^T - 2 \times \frac{1}{2} \bm U_k \bm U_k^T + \frac{1}{4} \bm I)) \\
        & = \frac{1}{4} {\rm tr}(\bm E^2) = \frac{1}{4} \varepsilon^2 , 
    \end{align*}
    we have 
    \begin{equation}
        \begin{split}
        &\| \bm E \bm U_k \bm U_k^T + \bm U_k \bm U_k^T \bm E - \bm E\| \\ 
        = & \sqrt{2(1+s)} \|\bm E (\bm U_k \bm U_k^T - \frac{1}{2} \bm I)\| = \sqrt{\frac{1+s}{2}} \varepsilon . 
        \end{split}
    \end{equation}
    Let $\bm Q = (\bm U_k \bm U_k^T - \bm I / 2)$ and 
    $s$ can be reformulated as 
    \begin{equation}
        s = \cos (\bm E \bm Q, \bm Q \bm E) , 
    \end{equation}
    and we have the following definition
    \begin{equation}
        \theta_* = \arccos(s) = \theta(\bm E \bm Q, \bm Q \bm E) . 
    \end{equation}
    According to Lemma \ref{lemma_commute}, 
    Assumption \ref{assumption_commute} indicates that 
    \begin{equation}
        \bm E (\bm U_k \bm U_k^T - \frac{1}{2} \bm I) = (\bm U_k \bm U_k^T - \frac{1}{2} \bm I) \bm E 
        \Rightarrow s < 1. 
    \end{equation}
    And therefore, $\sqrt{(1 + s )/2} \in [0, 1)$. 
    Let $\delta = 1 - \sqrt{(1 + s)/2} = 1 - \cos (\theta_* / 2) > 0$ and the above equation can be reformulated as 
    \begin{equation}
        \| \bm E \bm U_k \bm U_k^T + \bm U_k \bm U_k^T \bm E - \bm E\| = (1 - \delta) \varepsilon .
    \end{equation}
    The second term can be formulated as 
    \begin{align}
        \|\bm E \bm U_k \bm \Sigma_k^{-2} \bm U_k^T \bm E\| \leq & \|\bm U_k \bm \Sigma_k^{-2} \bm U_k^T\| \varepsilon^2 \\
        = & \varepsilon^2 \sqrt{{\rm tr}(\bm U_k \bm \Sigma_k^{-4} \bm U_k^T)} \\ 
        = & \varepsilon^2 \sqrt{{\rm tr}(\bm \Sigma^{-4} \bm U_k^T \bm U_k)} \\
        = & \varepsilon^2 \sqrt{\sum_{i=1}^k \frac{1}{\sigma_i^{4}}} \leq \sqrt{r}\frac{\varepsilon^2}{\sigma_{k}^2} ,
    \end{align}
    while the third term is  
    \begin{equation}
        \|\bm U_e \bm \Sigma_e^2 \bm U_e^T\| = \|\bm \Sigma_e^2\| = (\sum_{i=k+1}^n \sigma_i^4)^{1/2} \leq \sqrt{n-k} \sigma_*^2 . 
    \end{equation}
    To sum up, the error of $\mathcal{M}$ is bounded as 
    \begin{align} \label{eq_bound_full_rank}
        & \|\bm H \bm H^T - \bm P\| \leq (1-\delta)\varepsilon + \sqrt{k} \frac{\varepsilon^2}{\sigma_{k}^2} + \sqrt{e} \sigma_*^2 .
    \end{align}
    If 
    \begin{equation}
        \sqrt{k}\frac{\varepsilon}{\sigma_{k}^2} \leq \frac{\delta}{2} 
        \Rightarrow \varepsilon \leq \frac{\delta \sigma_{k}^2}{2 \sqrt{k}} 
    \end{equation}
    and 
    \begin{equation}
        \sqrt{n-k} \sigma_*^2 \leq \frac{\delta}{2} \Rightarrow \sigma_* \leq \sqrt{\frac{\delta}{2 \sqrt{n-k}}} , 
    \end{equation}
    then $\|\bm H \bm H^T - \bm P\| \leq \varepsilon$. 
    In other words, when $\varepsilon \leq \mathcal{O}(\delta)$ and $\sigma_* \leq \mathcal{O}(\sqrt{\delta})$, 
    the error will be bounded by $\varepsilon$. 

    For the case that ${\rm rank}(\bm X) = r < k$, it is not hard to verify that 
    \begin{align}
        \notag & \|\bm H \bm H^T - \bm P\| \\
        \leq & \| \bm E \bm U_r \bm U_r^T + \bm U_r \bm U_r^T \bm E - \bm E\| + \|\bm E \bm U_r \bm \Sigma_r^{-2} \bm U_r^T \bm E\| \\ 
        \leq & (1-\delta) \varepsilon + \sqrt{r} \frac{\varepsilon^2}{\sigma_r^2}. 
    \end{align}
    As 
    \begin{equation}
        \sqrt{r}\frac{\varepsilon}{\sigma_{r}^2} \leq \frac{\delta}{2} 
        \Rightarrow \varepsilon \leq \frac{\delta \sigma_{r}^2}{2 \sqrt{r}} = \mathcal{O}(\delta) 
    \end{equation}
    and $\sigma_* = 0 \leq \mathcal{O}(\sqrt{\delta})$, we get $\|\bm H \bm H^T - \bm P\| \leq \varepsilon$. 
    Hence, we have 
    \begin{equation}
        \min_{\bm W} \ell(\bm P, \bm X, \bm W) \leq \ell(\bm P, \bm X, \bm W_0) \leq \varepsilon, 
    \end{equation}
    and the theorem is proved. 
\end{proof}

\subsection{Proof of Theorem \ref{theo_assumption_fail}} \label{appendix_proof_2}

\begin{proof}
    According to Ineq. (\ref{eq_raw_bound_full_rank}), 
    \begin{equation}
        \begin{split}
        \|\bm H \bm H^T - \bm P\| \leq \| & \bm E \bm U_k \bm U_k^T + \bm U_k \bm U_k^T \bm E - \bm E\| \\ 
        & + \|\bm E \bm U_k \bm \Sigma_k^{-2} \bm U_k^T \bm E\| + \|\bm U_e \bm \Sigma_e^2 \bm U_e^T\| .
        \end{split}
    \end{equation}
    Suppose that $\bm E = \bm U \bm \Lambda \bm U^T$ so that $\bm P = \bm U (\bm S + \bm \Lambda) \bm U^T$ 
    where $\bm S = \bm \Sigma \bm \Sigma^T = {\rm diag}(\bm \Sigma_k^2; \bm 0)$. So, we have
    \begin{align*}
        \bm P & = \bm U (\bm S + \bm \Lambda) \bm U^T \\
        & = [\bm U_r, \bm U_e] 
        \left [
        \begin{array}{c c}
            \bm \Sigma_r^2 + \bm \Lambda_r & \bm 0 \\
            \bm 0 & \bm \Sigma_e^2 + \bm \Lambda_e
        \end{array}
        \right ]
        \left [
        \begin{array}{c}
            \bm U_r^T \\
            \bm U_e^T 
        \end{array}
        \right ] \\
        & = \bm U_k (\bm \Sigma_k^2 + \bm \Lambda_k) \bm U_k^T + \bm U_e (\bm \Sigma_e^2 + \bm \Lambda_e) \bm U_e^T . 
    \end{align*}
    Let $\bm V_k \bm W = (\bm \Sigma_k^3 + \bm \Lambda_k \bm \Sigma_k)^\dag (\bm \Sigma_k^2 + \bm \Lambda_k)^{1/2}$ and we have 
    \begin{align*}
            & \|\bm P - \bm H \bm H^T\| \\
            = & \|\bm U (\bm S + \bm \Lambda) \bm \Sigma \bm V^T \bm W \bm W^T \bm V \bm \Sigma^T (\bm S + \bm \Lambda) \bm U^T - \bm P\| \\
            = & \|\bm U_k (\bm \Sigma_k^3 + \bm \Lambda_k \bm \Sigma_k) \bm V_k^T \bm W \bm W^T \bm V_k (\bm \Sigma_k^3 + \bm \Lambda_k \bm \Sigma_k) \bm U_k^T - \bm P\| \\
            = & \|\bm U_k \hat{\bm I} (\bm \Sigma_k^2 + \bm \Lambda_k) \hat{\bm I} \bm U_k^T - \bm U_k (\bm \Sigma_k^2 + \bm \Lambda_k) \bm U_k^T \\ 
            & - \bm U_e (\bm \Sigma_e^2 + \bm \Lambda_e) \bm U_e^T \| \\
            = & \|\bm U_e (\bm \Sigma_e^2 + \bm \Lambda_e) \bm U_e^T\|
            \leq \|\bm \Sigma_e^2\| + \|\bm \Lambda_e\| \\
            \leq &  \varepsilon + \mathcal{O}(\sigma_*^2) ,
    \end{align*}
    where $\hat{\bm I} = (\bm \Sigma_k^3 + \bm \Lambda_k \bm \Sigma_k) (\bm \Sigma_k^3 + \bm \Lambda_k \bm \Sigma_k)^{\dag} = \bm \Sigma_k (\bm \Sigma_k^2 + \bm \Lambda_k) (\bm \Sigma_k^2 + \bm \Lambda_k)^\dag \bm \Sigma_k^{-1}$. 
    Clearly, $\hat{\bm I} = \mathbbm{1} \{\bm \Sigma_r^2 + \bm \Lambda_r \neq 0\}$. Therefore, 
    Hence, the theorem is proved. 
\end{proof}

\begin{corollary}
    If Assumption \ref{assumption_commute} does not hold and ${\rm rank}(\bm X) \leq k$, 
    then there exists $\bm W \in \mathbb{R}^{d \times k}$ so that $\|\bm P - \bm H \bm H\| \leq \varepsilon$.  
\end{corollary}
\begin{proof}
    Suppose that $\bm E = \bm U \bm \Lambda \bm U^T$ so that $\bm P = \bm U (\bm S + \bm \Lambda) \bm U^T$ 
    where $\bm S = \bm \Sigma \bm \Sigma^T = {\rm diag}(\bm \Sigma_r^2; \bm 0)$. 
    Then we have 
    \begin{equation}
        \begin{split}
        \bm P &= \bm U (\bm S + \bm \Lambda) \bm U^T = [\bm U_r, \bm U_e] 
        \left [
        \begin{array}{c c}
            \bm \Sigma_r^2 + \bm \Lambda_r & \bm 0 \\
            \bm 0 & \bm \Lambda_e
        \end{array}
        \right ]
        \left [
        \begin{array}{c}
            \bm U_r^T \\
            \bm U_e^T 
        \end{array}
        \right ] \\
        & = \bm U_r (\bm \Sigma_r^2 + \bm \Lambda_r) \bm U_r^T + \bm U_e \bm \Lambda_e \bm U_e^T .
        \end{split}
    \end{equation}
    Let $\bm W = [\bm W_r; \bm 0]$ subjected to $\bm V_r^T \bm W_r = (\bm \Sigma_r^3 + \bm \Lambda_r \bm \Sigma_r)^{\dag} (\bm \Sigma^2 + \bm \Lambda_r)^{1/2}$. 
    Then 
    \begin{align*}
        & \|\bm H \bm H^T - \bm P\| \\
        = & \|\bm U (\bm S + \bm \Lambda) \bm \Sigma \bm V^T \bm W \bm W^T \bm V \bm \Sigma^T (\bm S + \bm \Lambda) \bm U^T - \bm P\| \\
        = & \|\bm U_r (\bm \Sigma_r^3 + \bm \Lambda_r \bm \Sigma_r) \bm V_r^T \bm W \bm W^T \bm V_r (\bm \Sigma_r^3 + \bm \Lambda_r \bm \Sigma_r) \bm U_r^T - \bm P\| \\
        = & \|\bm U_r \hat{\bm I} (\bm \Sigma_r^2 + \bm \Lambda_r) \hat{\bm I} \bm U_r^T - \bm U_r (\bm \Sigma_r^2 + \bm \Lambda_r) \bm U_r^T - \bm U_e \bm \Lambda_e \bm U_e^T \| ,
    \end{align*}
    where $\hat{\bm I} = (\bm \Sigma_r^3 + \bm \Lambda_r \bm \Sigma_r) (\bm \Sigma_r^3 + \bm \Lambda_r \bm \Sigma_r)^{\dag} = \bm \Sigma_r (\bm \Sigma_r^2 + \bm \Lambda_r) (\bm \Sigma_r^2 + \bm \Lambda_r)^\dag \bm \Sigma_r^{-1}$. 
    Clearly, $\hat{\bm I} = \mathbbm{1} \{\bm \Sigma_r^2 + \bm \Lambda_r \neq 0\}$. Therefore, 
    \begin{equation}
        \|\bm H \bm H^T - \bm P\| = \| \bm U_e \bm \Lambda_e \bm U_e^T\| = \|\bm \Lambda_e\| \leq \varepsilon. 
    \end{equation}
    Hence, the corollary is proved. 
\end{proof}

\bibliographystyle{IEEEtran}
\bibliography{SGNN.bib}

\begin{IEEEbiography}[{\includegraphics[width=1in,height=1.25in,clip,keepaspectratio]{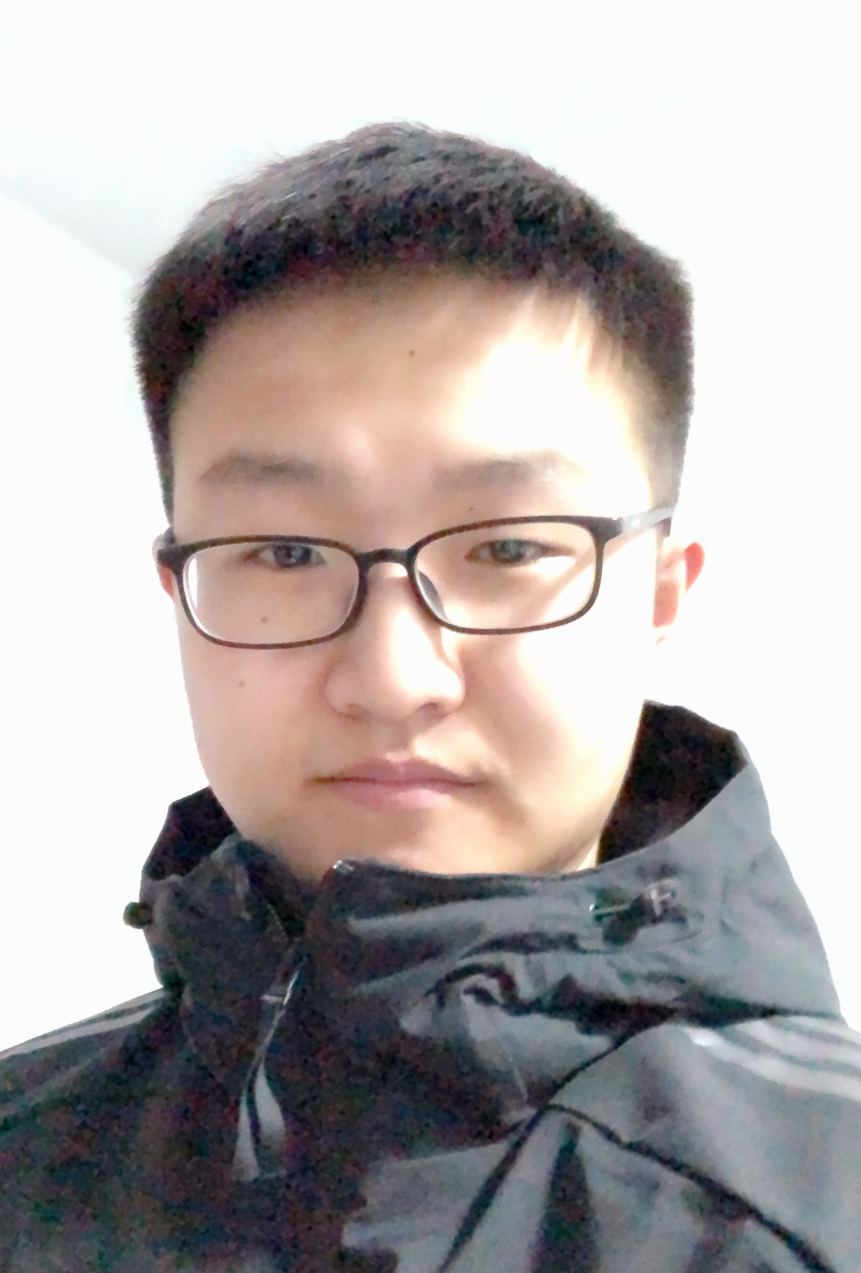}}]{Hongyuan Zhang}
    received the B.E. degree in software engineering from Xidian Unibersity, Xi'an, China in 2019 
    and received the Ph.D. degree from the School of Computer Science and the School of Artificial Intelligence, Optics and Electronics (iOPEN), Northwestern Polytechnical University, Xi'an, China in 2024. 
\end{IEEEbiography}

\begin{IEEEbiography}[{\includegraphics[width=1in,height=1.25in,clip,keepaspectratio]{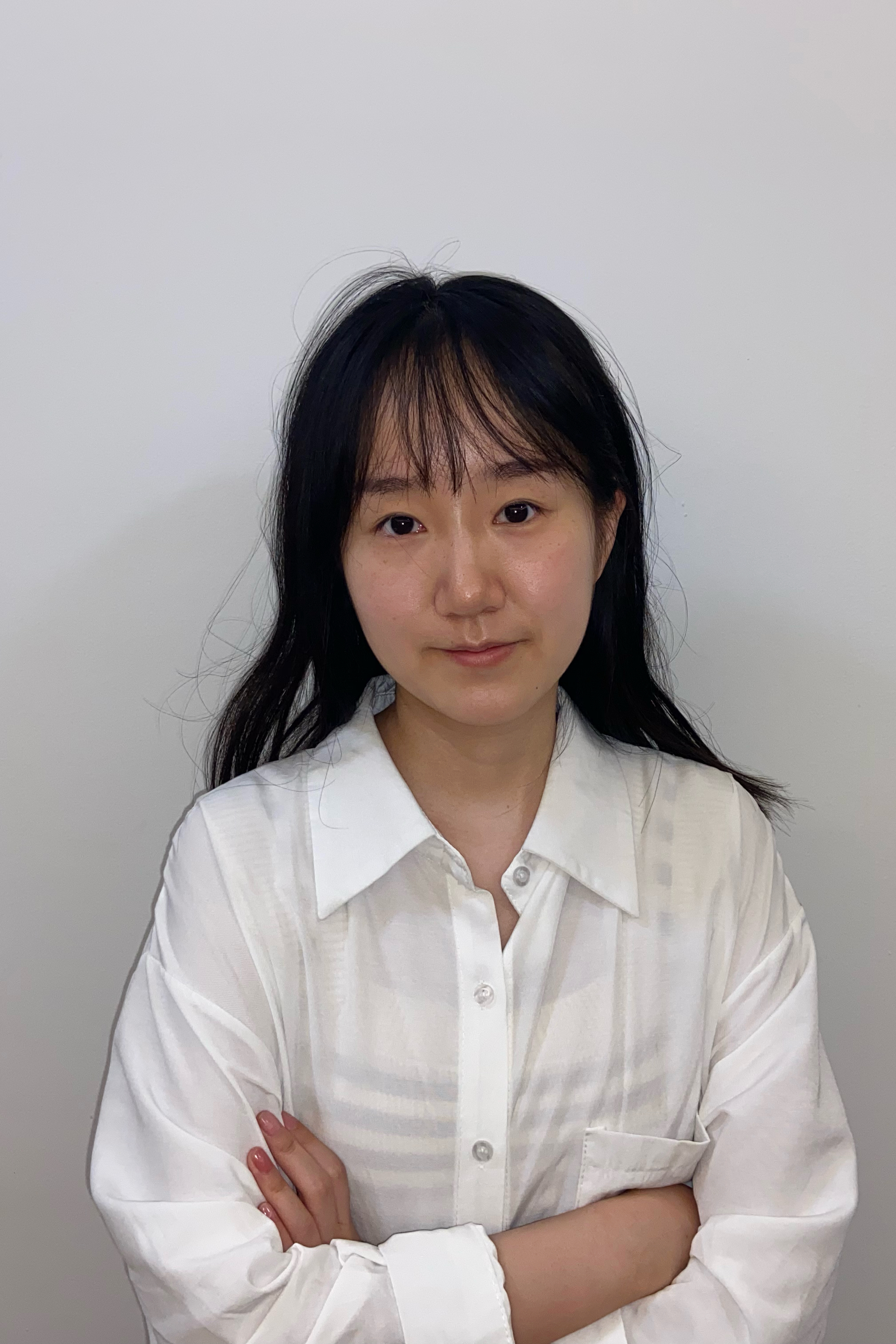}}]{Yanan Zhu}
    received the B.E. degree in computer science and technology from Nanjing University of Aeronautics and Astronautics, Nanjing, China, in 2021 
    and received the master's degree with School of Computer Science and School of Artificial Intelligence, Optics and Electronics (iOPEN), Northwestern Polytechnical University, Xi'an, China in 2024.
\end{IEEEbiography}

\begin{IEEEbiographynophoto}{Xuelong Li} (M'02-SM'07-F'12) 
    is the Chief Technology Officer (CTO) and Chief Scientist of the China Telecom. 
\end{IEEEbiographynophoto}

\end{document}